\documentclass{article}

    \usepackage[nonatbib,preprint]{neurips_2020}

\usepackage[utf8]{inputenc} 
\usepackage[T1]{fontenc}    
\usepackage{url}            
\usepackage{nicefrac}       
\usepackage{microtype}      

\usepackage{xr}
\usepackage{booktabs}
\usepackage{color}
\usepackage{latexsym}              
\usepackage{amsmath}               
\usepackage{amssymb}               
\usepackage{amsfonts}              
\usepackage{amsthm}                
\usepackage{accents}
\usepackage{mathtools}
\usepackage{multirow}
\usepackage{tikz}                  
\usepackage{bbm}
\usepackage{cite}
\usepackage{subcaption}
\usepackage{adjustbox}
\usepackage{listings}

\usetikzlibrary{arrows,positioning,shapes}
\usepackage{pifont}
\usepackage[ruled,vlined]{algorithm2e}
\usepackage{enumitem}
\usetikzlibrary{matrix}
\usepackage{algorithmic}

\usepackage[mathcal]{eucal}

\definecolor{codegreen}{rgb}{0,0.6,0}
\definecolor{codegray}{rgb}{0.5,0.5,0.5}
\definecolor{codepurple}{rgb}{0.58,0,0.82}
\definecolor{backcolour}{rgb}{0.95,0.95,0.92}

\lstdefinestyle{mystyle}{
    backgroundcolor=\color{backcolour},   
    commentstyle=\color{codegreen},
    keywordstyle=\color{magenta},
    numberstyle=\tiny\color{codegray},
    stringstyle=\color{codepurple},
    basicstyle=\ttfamily\footnotesize,
    breakatwhitespace=false,         
    breaklines=true,                 
    captionpos=b,                    
    keepspaces=true,                 
    numbers=left,                    
    numbersep=5pt,                  
    showspaces=false,                
    showstringspaces=false,
    showtabs=false,                  
    tabsize=2
}

\usepackage{cleveref}
\crefname{assumption}{Assumption}{Assumptions}
\crefname{equation}{Eq.}{Eqs.}
\crefname{figure}{Figure}{Figures}
\crefname{table}{Table}{Tables}
\crefname{section}{Section}{Sections}
\crefname{algorithm}{Algorithm}{Algorithms}
\crefname{theorem}{Theorem}{Theorems}
\crefname{lemma}{Lemma}{Lemmas}
\crefname{proposition}{Proposition}{Propositions}
\crefname{corollary}{Corollary}{Corollaries}
\crefname{example}{Example}{Examples}
\crefname{appendix}{Appendix}{Appendixes}
\crefname{remark}{Remark}{Remark}

\newcounter{remark}[section]

\newcommand{\calX}{\mathcal{X}}
\newcommand{\calY}{\mathcal{Y}}
\newcommand{\calZ}{\mathcal{Z}}
\newcommand{\calE}{\mathcal{E}}

\newcommand{\calH}{\mathcal{H}}

\newcommand{\calG}{\mathcal{G}}

\newcommand{\calC}{\mathcal{C}}

\newcommand{\xfun}{{\bf x}}
\newcommand{\xfunhat}{\widehat {\bf x}}

\newcommand{\Ghat}{\widehat G}
\newcommand{\Gstar}{G^\star}

\newcommand{\xp}{\text{x}}
\newcommand{\yp}{\text{y}}
\newcommand{\zp}{\text{z}}

\newcommand{\gfun}{g}

\DeclareMathAlphabet{\mathsfsl}{OT1}{cmss}{m}{sl}

\renewcommand{\phi}{\varphi}

\newcommand{\Rspace}[1]{\mathbb{R}^{#1}}

\newcommand{\R}{\mathbb{R}}
\newcommand{\N}{\mathbb{N}}

\newcommand{\gstar}{g^\star}

\newcommand{\ghat}{\widehat{g}}

\newcommand*{\defeq}{\mathrel{\vcenter{\baselineskip0.5ex \lineskiplimit0pt
                     \hbox{\scriptsize.}\hbox{\scriptsize.}}}%
                     =}

\newcommand{\eqal}[1]{
\begin{align}
#1
\end{align}
}

\newcommand{\argmin}{\operatorname*{arg\; min}}
\newcommand{\argmax}{\operatorname*{arg\; max}}

\newcommand{\Expect}{\operatorname{\mathbb{E}}}

\newcommand{\Lp}{\ell}

\theoremstyle{plain}  
\newtheorem{theorem}{Theorem}[section]

\newtheorem{lemma}[theorem]{Lemma}
\newtheorem{proposition}[theorem]{Proposition}

\newcommand{\smui}{\sum_{i=1}^n \mu_i}

\title{Structured and Localized Image Restoration}

\author{%
  Thomas Eboli$^\ddagger$\thanks{Equal contribution} \\
  \texttt{thomas.eboli@inria.fr} \\
  \And
  Alex Nowak-Vila$^{\ddagger *}$ \\
  \texttt{alex.nowak-vila@inria.fr} \\
  \And
  Jian Sun\thanks{Xi'an Jiaotong University}\\
  \texttt{jiansun@xjtu.edu.cn} \\
  \And
  Francis Bach\thanks{INRIA, D\'epartement d’informatique de l’ENS, ENS, CNRS, PSL University}\\
  \texttt{francis.bach@inria.fr} \\
  \And
  Jean Ponce$^\ddagger$\\
  \texttt{jean.ponce@inria.fr} \\
  \And
  Alessandro Rudi$^\ddagger$\\
  \texttt{alessandro.rudi@inria.fr} \\
}

\begin{document}

\maketitle

\begin{abstract}
    We present a novel approach to image restoration that leverages ideas from localized structured prediction and non-linear multi-task learning.
    We optimize a penalized energy function regularized by a sum of terms measuring the distance between patches to be restored and clean patches from an external database gathered beforehand.
    The resulting estimator comes with strong statistical guarantees leveraging local dependency properties of overlapping patches.
    We derive the corresponding algorithms for energies based on the mean-squared and Euclidean norm errors. 
    Finally, we demonstrate the practical effectiveness of our model on different image restoration problems using standard benchmarks.
\end{abstract}

\section{Introduction}

After decades of work, image processing is still a vibrant research field, particularly well suited to modern machine learning technology, given that it is often easy to generate clean/corrupted image pairs, and more pertinent than ever with the ubiquitous use of smartphone cameras, and many applications in personal photography \cite{abdelhamed18sidd}, microscopy \cite{weigert18microscopy} and astronomy \cite{starck06astronomical}, for example.
Given a clean image $x$, a known (in general) linear operator $B$ such as downsampling or blur, and additive noise $\varepsilon$ with standard deviation $\sigma$, the degraded image $y$ is typically modelled as
\begin{equation}\label{eq:formationmodel}
    y = Bx + \varepsilon.
\end{equation}
The operator $B$ is problem specific, e.g., for denoising it is the identity, and for deblurring it is a sparse matrix standing for a convolution with some blur kernel\cite{mairal14sparse}.
Solvers for this problem date back from Wiener's seminal work \cite{wiener49extrapolation}, that cast image restoration as the minimization of some mean-squared error (MSE).
Decades later, Rudin {\em et al.} \cite{rudin92totalvariation} proposed 
a variational formulation of image denoising using both a data term and a total variation prior.
A discretized version of this approach and its optimization were proposed by Wang {\em et al.} \cite{wang08alternating} and it has since been extended to various tasks such as non-blind deblurring using much powerful priors \cite{krishnan09hyperlaplacian, zoran11learning}.

In this work, we propose a non-parametric general image restoration algorithm based on localized structured prediction \cite{ciliberto19localized, ciliberto20general} and non-linear multi-task learning \cite{Ciliberto17multitask}.
In particular, we adapt to patch-based image restoration the theoretical framework of \cite{ciliberto19localized} for energy minimization under local task-specific constraints.
Our approach is reminiscent of that of Zoran and Weiss \cite{zoran11learning} and its patch-based energy function, but it also combines ideas from example-based \cite{freeman02example, he14inpainting}, variational \cite{krishnan09hyperlaplacian, zoran11learning} and data-driven methods \cite{takeda07kernel, chen17trainable} into a single model. A crucial and distinctive feature of our approach is that we have a complete statistical analysis of the corresponding prediction error. In particular, we provide an interpretable error upper bound explicitly parameterized by quantities depending on local dependency properties of the natural images. 
Most image restoration works do not provide a theoretical analysis, with notable exceptions such as \cite{mairal10online} in the context of online dictionary learning.
Our estimator is based on a convex energy function that can be minimized exactly, contrary to standard example-based methods that often rely on approximation of Markov random fields \cite{freeman02example} or local and non-local averaging techniques \cite{buades05nonlocal, dabov07image}.
Data-driven methods such as sparse coding \cite{elad06image, zeyde10scaleup} often assume an unrealistic i.i.d. patch model, whereas our framework explicitly accounts for the patch correlations.
Although CNN-based methods achieve state-of-the-art results in many restoration tasks \cite{zhang17beyond, dong16srcnn, zhang17fullyconvolutional},
they often ignore the underlying image degradation model, require a very large number of training data and suffer from a lack of interpretability and theoretical guarantees.
Instead, the proposed approach only requires a limited number of examples since it leverages the interplay of the the forward model with the local properties of the data at the patch level, and it is fully interpretable.
Finally, our approach is highly modular, and many of its components can be changed and/or learnt while ensuring convergence, similar to Ciliberto {\em et al.} \cite{ciliberto19localized, ciliberto20general} in the context of statistical learning theory.
Changing these components shapes the form of the energy and can lead to specific solvers for each situation. 
Concretely, our contributions include:
\begin{itemize}[leftmargin=1cm]
    \item A new image restoration framework based on localized structured prediction \cite{ciliberto19localized} and non-linear multi-task learning \cite{Ciliberto17multitask} that bridges statistical learning theory and image restoration. The estimator is written as a minimizer of an energy with a learnable prior defined on patches.
    \item A theoretical analysis of the estimator with explicit quantities depending on the local dependency properties of the problem.
    \item Two efficient implementations: For the MSE-based solver, we solve a linear system with conjugate gradient. For the Euclidean norm-based solver, we introduce a splitting scheme that alternates between conjugate gradient and dual coordinate ascent.
    \item An experimental validation of the methods on standard benchmarks for non-blind deblurring and upsamling. We achieve results comparable or better than standard variational methods.
    Our goal is not to beat the state of the art but to showcase the practical abilities of this general framework to solve image restoration problems.
\end{itemize}

\textbf{Related work.}
Image restoration methods can be divided into three groups.

\textit{Example-based methods} restore patches based on similar ones taken from an external dataset \cite{freeman02example, sun10contextconstrained} or the image at hand by exploiting self similarities in a given neighborhood \cite{dabov07image, buades05nonlocal, katkovnik10local} or at different scales \cite{ glasner09srsingleimage}.
For example, \cite{freeman02example} addresses image upsampling by comparing low-resolution patches in the target image with those in an external dataset retrieving the nearest neighbors, and copying and pasting their high-resolution version into the restored image.
In \cite{dabov07image}, the authors group and average similar patches in the image to remove noise.

\textit{Energy-based methods} minimize an objective function typically composed of a fitting term enforcing \cref{eq:formationmodel} and a penalty term favoring solutions exhibiting features of natural images.
For instance, the total variation prior \cite{rudin92totalvariation, wang08alternating, krishnan09hyperlaplacian} favors images with sharp edges, while the prior on patches of \cite{zoran11learning} uses the centroids of a Gaussian mixture model (GMM) to guide the restoration. 
A drawback of these methods is that they rely on handcrafted priors which might not capture all aspects of natural images.

\textit{Learning-based methods} minimize an empirical risk based on a dataset of pairs of clean and degraded images. In particular, dictionary learning methods estimate a small number of atoms to encode sparse models of patches and can be used in tasks such as denoising \cite{elad06image, mairal09nonlocal, mairal14sparse} and upsampling \cite{zeyde10scaleup}.
Most learning-based models are parametric (see \cite{takeda07kernel, chatterjee12patch} for counter examples) and learnt in a supervised manner, often as convolutional neural networks (CNNs), especially in problems like denoising \cite{zhang17beyond}, upsampling \cite{dong16srcnn} or deblurring \cite{zhang17beyond}.
Recently, data-driven and energy-based methods have also been combined to learn the minimizer of energies with tunable parameters in the regularizer \cite{zhang17fullyconvolutional, chen17trainable}.

\begin{figure}[t!]
    \centering
    \includegraphics[width=0.50\textwidth]{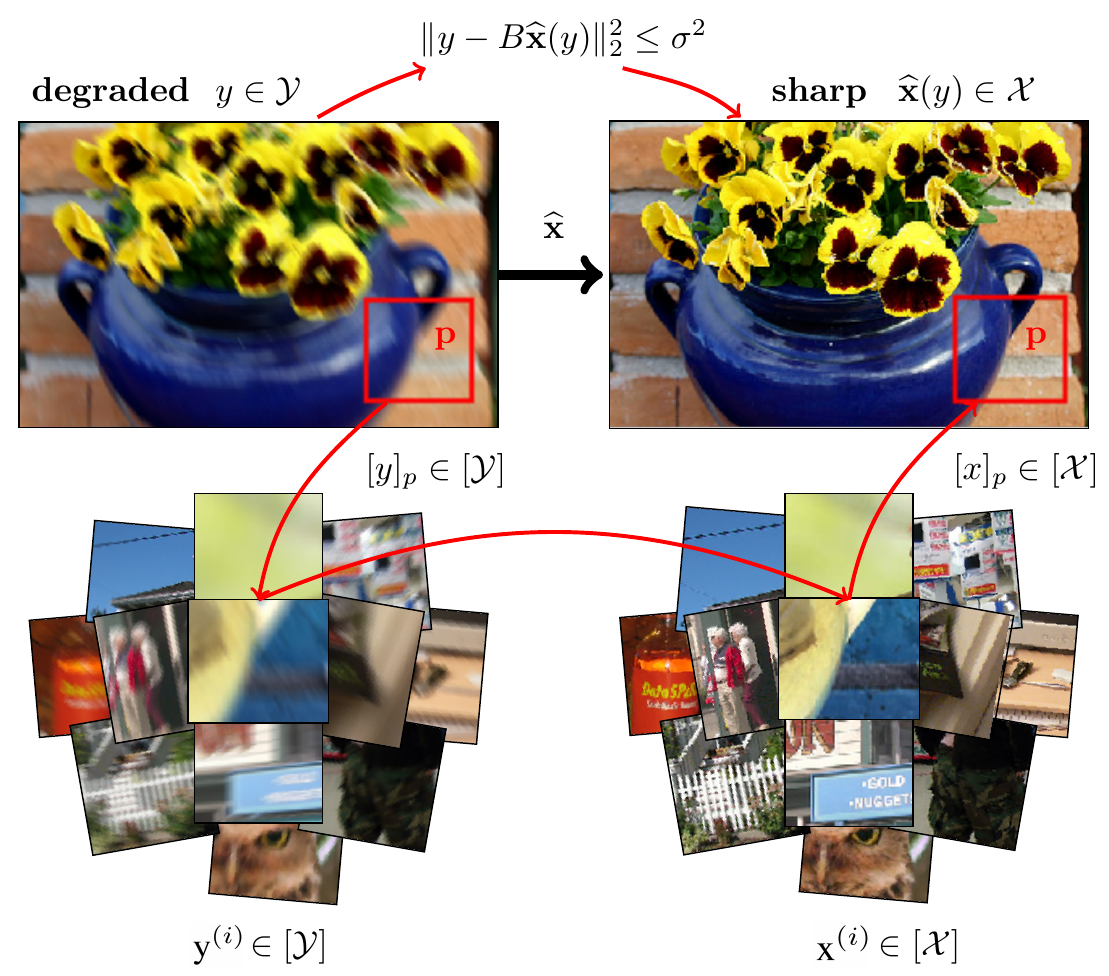}
    \caption{
    Given a degraded (blurry here) image $y$, the estimator $\xfunhat$ predicts a sharp image $\xfunhat(y)$. 
    The vector of weights $\alpha$ is computed by comparing the patches $y_p$ from the image to the patches $\yp^{(i)}$ from an external set, associated with clean versions $\xp^{(i)}$. 
    These weights are used to construct an energy function that favors sharp patches $x_p$ similar to samples $\xp^{(i)}$. 
    Finally, the estimated sharp image is a minimizer of this energy, penalized with the formation model of \cref{eq:formationmodel}.}
    \label{fig:diagram}
\end{figure}

\section{Proposed framework}\label{sec:localimagerestoration}
Let $\calX$ be the space of natural images and $\calY$ the space of degraded ones.
The goal of image restoration is to estimate a function ${\bf x}:\calY\rightarrow\calX$ that computes the sharp image ${\bf x}(y)$\footnote{We will denote by ${\bf x}$ functions from $\calY$ to $\calX$ and by $x,y$ elements of $\calX, \calY$.} from a degraded image $y$. 
Concretely, the degradation process is modeled in terms of a probability distribution $\rho$ on $\calX \times \calY$, and the goal is to find the function~$\xfun^\star:\calY\xrightarrow{}\calX$ that minimizes the expected risk
\begin{equation}\label{eq:expectedrisk}
  \calE(\xfun) = \Expect_{(y,x)\sim\rho}~ L(\xfun(y),x),
\end{equation}
using only the observed dataset $(y^{(i)}, x^{(i)})_{i=1}^n$ sampled from $\rho$, where $L$ is a loss between images. 
The formulation above is very general and standard to supervised learning problems \cite{vapnik2013nature}. Here, instead, we want to specialize the learning process to leverage the properties of the image formation model. 
By enforcing \eqref{eq:formationmodel}, we can restrict the class of estimators to functions $\xfun:\calY \to \calX$ satisfying $\|y-B \xfun(y)\|_2^2\leq\sigma^2$ for all $y \in \calY$. 
This is equivalent \cite{Ciliberto17multitask} to  
\begin{equation}\label{eq:estimatorsigma}
    \xfun^\star_\sigma(y) = \argmin_{\|y - Bx\|^2 \leq \sigma^2}~\Omega_{L}^{\pi}(x~|~y), \hspace{0.3cm}\text{where}\hspace{0.3cm} \Omega_{L}^{\pi}(x~|~y) = \mathbb{E}_{x'\sim\pi_y}L(x,x'),
\end{equation}
where $\pi_y(x) = \rho(x~|~y)$ is the conditional distribution of $x$ given $y$. Note that the conditional distribution $\pi$ is unknown, and so the quantity $\Omega_{L}^{\pi}(x~|~y)$ cannot in general be computed exactly.
By Lagrangian duality, we can rewrite for a suitable $\gamma \geq 0$ the optimization problem as
\begin{equation}\label{eq:exactestimator}
    \argmin_{x\in\calX}~\Omega_L^{\pi}(x~|~y) + \frac{\gamma}{2}\|y - Bx\|_2^2.
\end{equation}

\subsection{Leveraging local priors: localized structured prediction}
\label{sec:localpriors}
We can further specialize the class of estimators considered by leveraging the strong structural properties at a local level that are common in natural images \cite{mairal14sparse}.
Let us consider decompositions of the images $x$ and $y$ into overlapping patches indexed by $p\in P$, where $P$ is the set of patch identifiers. Denote by $x_p$ and $y_p$ the $p$-th patch of the sharp image and degraded image, and denote by $[\calX], [\calY]$ the corresponding spaces of patches. 
In vector form, $x_p$ and $y_p$ correspond to a subset of coordinates of $x, y$. We assume that the image degradation problem is \textit{local}, i.e., $y_p$ is determined by $x_p$ and~$B$. Thus, estimating $x$ from $y$ can be seen as a multi-task learning problem, where the function~$y \mapsto x$ is factored in a group of functions $y_p \mapsto x_p$ collaboratively solving the global problem. 
We similarly assume that the degradation function is stationary with respect to the position of the patch in the image, that is, the functions $y_p \mapsto x_p$ are independent of the patch identifier $p$. As in \cite{ciliberto19localized} we formalize 
the behavior described above as follows.

{\bf Between-locality assumption:}
$x_p$ is conditionally independent from $y$ given $y_p$. 
Moreover the following holds for all patches $p, p^\prime\in P$,
\begin{equation}\label{eq:between-locality-stationarity1}
    \pi_{y_p}(x_p) = \pi_{y_{p^\prime}}(x_{p^\prime}).
\end{equation}
Note that assumption \eqref{eq:between-locality-stationarity1} uses the stationarity mentioned above, as the conditional distribution of patches $\pi_{y_p}(x_p)$ between $x_p$ and $y_p$ is independent of the index $p$.
It is then natural to consider losses $L$ in \eqref{eq:exactestimator} that measure the global error between two images in terms of errors on the single patches
\begin{equation}\label{eq:globalerroraspatcherrors}
    \textstyle{L(x,x^\prime) = \frac{1}{|P|}\sum_{p \in P} \ell(x_p, x^\prime_p)},
\end{equation}
where $\ell$ is a loss function between patches. For instance, one can take $\ell$ to be the mean-squared error $\|x_p-x_p'\|_2^2$, which decomposes additively over the pixels of the pair of patches, making the predictions independent for every pixel of the image. If one considers instead the non-squared Euclidean norm $\|x_p-x_p'\|_2$, then the problem is solved by taking into account inter-dependencies between overlapping patches \cite{jacob09group}.
Given assumption \eqref{eq:between-locality-stationarity1}, it is convenient to drop the patch index and designate patches without specifying that they come from a certain location of an image. In the rest of the paper we will denote by $\xp$ and $\yp$ two generic patches \footnote{We warn the reader about the similarity of notation between generic patches $\xp,\yp$ and full images $x,y$.} belonging to $[\calX]$ and $[\calY]$, respectively. Note that with this construction, the regularizer decomposes additively as
\begin{equation}
    \textstyle \Omega_L^{\pi}(x~|~y) = \frac{1}{|P|}\sum_{p \in P} \Omega_{\ell}^{\pi}(x_p~|~ y_p),\quad \textrm{where}, \quad \Omega_{\ell}^{\pi}(\xp~|~\yp) = \Expect_{\xp'\sim\pi_{\yp}}\ell(\xp,\xp'),
\end{equation}
for all $ \xp\in[\calX], \yp \in [\calY]$. Then, our problem takes the form
\eqal{\label{eq:regression function}
\textstyle \argmin_{x\in\calX}~\frac{1}{|P|}\sum_{p \in P}\Omega_{\ell}^{\pi}(x_p~|~y_p) + \frac{\gamma}{2}\|y - Bx\|_2^2,
}
enforcing at the same time prior knowledge of the problem at the global level via the variational representation of the formation model \eqref{eq:formationmodel} and at the local level via the data-driven regularizer on patches. However, as we do not have access to the conditional distribution $\pi_{\yp}(\xp)$, this term cannot be computed exactly. We now show how to approximate it from supervisory data.

\subsection{Construction of the estimator}\label{subsec:construcitonoftheestimator}
The goal is to learn the function $\yp\mapsto \Omega_{\ell}^{\pi}(\xp~|~\yp)$ for every $\xp\in[\calX]$  from a dataset of $m$ pairs of patches $(\yp^{(i)}, \xp^{(i)})$ extracted from images degraded by $B$ and the associated clean images. We want to construct the estimator simultaneously for every sharp patch $\xp$. In order to do this, following \cite{ciliberto20general, ciliberto19localized}, we consider estimators that are linear combinations of the loss function evaluated at the data points
\begin{equation}\label{eq:omega1estimator}
 \textstyle{\Omega_{\ell}^{\pi}(\xp~|~\yp) \approx \sum_{i=1}^{m}\alpha_{i}(\yp)\Lp(\xp,\xp^{(i)})},
\end{equation}
where $\alpha_{i}(\yp)$ are scalar coefficients learnt from the dataset of patches. For any degraded patch $\yp$, $\alpha_{i}(\yp)$ can be interpreted as a measure of similarity between $\yp$ and the patch $\yp^{(i)}$ of the training dataset. 
Finally, we define the estimator $\xfunhat(y)$ as:
\begin{equation}\label{eq:finalestimator}
\xfunhat(y)\in\argmin_{x\in\calX}~
    \frac{1}{|P|}\sum_{p\in P}\sum_{i=1}^{m}\alpha_{i}(y_p)\Lp(x_p,\xp^{(i)}) + \frac{\gamma}{2}\|y-Bx\|_2^2.
\end{equation}
We provide two ways of computing the coefficients $\alpha_i$.

\textbf{Kernel ridge regression (KRR).} 
To derive the estimator we start from the observation that when we fix $\xp \in [\calX]$, then $\textstyle{\Omega_{\ell}^{\pi}(\xp~|~\yp)}$ is just a function from $[\calY]$ to $\R$, that we denote $g_{\xp}^\star: [\calY]\rightarrow\Rspace{}$ for convenience. The idea now is to learn $g_\xp^\star$ from data. In particular, note that $g^*_\xp$ is the minimizer over $\{g:[\calY]\rightarrow\Rspace{}\}$ of the least-squares expected risk $\Expect_{(\xp',\yp')\sim\rho}(g(\yp') - \ell(\xp, \xp'))^2$.
Then, $g_\xp^\star$ can be estimated from a dataset $(\yp^{(i)}, \ell(\xp, \xp^{(i)}))$ of $m$ pairs using regularized empirical risk minimization and standard techniques from supervised machine learning as kernel ridge regression \cite{scholkopf2002learning}. More specifically, let $\calG$ be a reproducing kernel Hilbert space (RKHS) with associated positive kernel $k:[\calY]\times[\calY]\rightarrow\Rspace{}$ over pairs of patches of degraded images \cite{aronszajn1950theory}. Then, for every $\xp \in [\calX]$, the function~$g_\xp^\star$ can be approximated by $\widehat{g}_\xp$ as
\begin{equation}\label{eq:rerm}
 \widehat{g}_\xp ~=~ \textstyle{\argmin_{g\in\calG}\frac{1}{m}\sum_{i=1}^m(g(\yp^{(i)}) - \ell(\xp, \xp^{(i)}))^2 + \frac{\lambda}{2}\|g\|_{\calG}^2},
\end{equation}
where $\|g\|_{\calG}$ denotes the norm of $g$ in ${\cal G}$ and $\lambda$ is a regularization parameter. The solution takes the form of~\eqref{eq:omega1estimator} where the vector $\alpha(\yp)$ is the solution of the linear system
\begin{equation}\label{eq:alphakrr}
(K  + m\lambda  I) \alpha (\yp) = v(\yp)\in \R^{m}.
\end{equation}
Here, $v(\yp)_i = k (\yp,\yp^{(i)})$ and
$K$ is a matrix in $\Rspace{m\times m}$ defined by~$K_{ij}~=~k (\yp^{(i)}, \yp^{(j)})$,

\textbf{Nadaraya-Watson (NW).}
Conversely, $\yp\mapsto \Omega^\pi_{\ell}(\xp, \yp)$ can also be estimated with the Nadaraya-Watson estimator \cite{nadaraya64estimating}, a version of kernel density estimation \cite{parzen1962estimation} for regression. We have
\begin{equation}\label{eq:alphanw}
    \alpha(\yp) = \frac{1}{\mathbbm{1}^\top v(\yp)} v(\yp)\in\Rspace{m},
\end{equation}
where $\mathbbm{1}$ is a vector of size $n$ with all entries equal to 1.
Contrary to \eqref{eq:alphakrr}, all coefficients are positive and their computation does not involve the kernel matrix $K$. 
In the following \cref{sec:stats}, we analyze the statistical properties of the estimator \eqref{eq:finalestimator} when the $\alpha$'s are computed using KRR.

\section{Statistical Guarantees}\label{sec:stats}

The goal of this section is twofold. First we show that the estimate $\xfunhat$ in \cref{eq:finalestimator} approaches the optimal estimator $\xfun^\star$ minimizing \eqref{eq:expectedrisk} when the size of the patch dataset tends to infinity. Moreover, we explicitly characterize the convergence rate at which this happens. For simplicity, and based on the analysis of \cite{ciliberto19localized}, we assume that the coefficients $\alpha$ are computed using kernel ridge regression~\eqref{eq:alphakrr}.

Let us first define the sampling scheme to generate patches from images.
We generate a dataset of clean and degraded pairs of patches $(\yp^{(i)}, \xp^{(i)})$ ($i=1,\dots,m$) by sub-sampling uniformly patches from $n$ pairs of clean and degraded-by-$B$ images.
We theoretically justify and validate experimentally why subsampling patches, which corresponds to $m\ll|P|n$, is a good choice.

In order to simplify the analysis, we work with the constrained version of the minimizer~$\xfunhat(y)$ defined in \cref{eq:finalestimator}, where the estimated regularizers \eqref{eq:omega1estimator} are minimized over the constraint~$\|y-Bx\|_2^2\leq \sigma^2$, instead of the penalized version.
Our main result is a bound on the excess risk $\calE(\xfunhat) - \calE(\xfun_\sigma^\star)$ depending on the number of images~$n$, the number of subsampled patches $m$, and interpretable constants describing local dependency properties of the problem. More specifically, the bound on the excess risk depends on the quantities $c_{B, \sigma}$ and $q$ defined as:
\begin{equation}
    c_{B,\sigma} = \frac{\Expect_{y\sim\rho_{\calY}}\operatorname{diam}(\calC_{B,\sigma}(y))}{\operatorname{diam}(\calX)}\leq 1,\quad 
     q=\frac{1}{|P|r^2}\sum_{p,p'\in P}\Expect_{y, y'}C_{p,p'}(y, y'),
\end{equation}
where $\operatorname{diam}$ stands for diameter; $\calC_{B,\sigma}(y) = \{x\in\calX~|~\|Bx-y\|_2\leq\sigma\}\subseteq\calX$ is the constraint set defined by the formation model \eqref{eq:formationmodel}; $C _{p,p'}(y,y') = k (y_p, y_{p'})^2 - k (y_p,y'_{p'})^2$ is a measure of similarity between $y$ and $y'$; and $r^2 = \sup_{\yp\in[\calY]}k(\yp,\yp)$ is the maximum squared norm of the features. We recall that $k$ is the kernel computed from the degraded patches used to train KRR (\cref{subsec:construcitonoftheestimator}). 
The constant $c_{B,\sigma}$ is a ratio between the expected diameter of the constraint set and the diameter of the full space of sharp images. The smaller the ratio, the more informative is the knowledge of the formation model \eqref{eq:formationmodel} for the learning task.
The constant $q$ measures the total correlation between patches of image $y$. 
It satisfies~$q\approx |P|$ if there is large correlation between patches and $q\approx 0$ if they are mostly independent. This is essentially the same quantity appearing in the analysis by \cite{ciliberto19localized}, 
which in that case is the total correlation of the parts of a generic structured object (see Example~1 of \cite{ciliberto19localized}).
To state the theorem we need a smoothness assumption on the loss and on the target distribution. Let the set of patches be defined respectively as $[\calX] = [-1,1]^{d_{\calX}\times d_{\calX}}$ and $[\calY] = \Rspace{d_{\calY}\times d_{\calY}}$, with $d_\calX, d_\calY \in \mathbb{N}$. Denote by $W^{s}_2(Z)$ the {\em Sobolev space} of smoothness $s > 0$ in a set $Z$, i.e., the space of functions with square-integrable weak derivatives up to order $s$ \cite{adams2003sobolev}. 

\noindent\textbf{Assumption on the loss $\boldsymbol{\ell}$.} Let $\ell_{\xp}$ be defined as $\xp'\mapsto\ell(\xp',\xp)$. There exists $C$ in $(0,\infty)$ and~$s\geq (d_{\calX}^2 +1)/2$ such that $\sup_{\xp \in [\calX]} \|\ell_{\xp}\|_{W^{s}_2([\calX])} \leq C$.

\noindent\textbf{Assumption on the target distribution.}
Let $m \in \N$ s.t. $W^m_2([\calY]) \subseteq {\calH}$ where $\calH$ is the reproducing kernel Hilbert space associated to the kernel $k$ \cite{aronszajn1950theory, berlinet2011reproducing}. For example, $m = d_{\calY}^2/2 + 1$ for the Laplacian kernel $k(\yp,\yp') = e^{-\|\yp-\yp'\|}$. 
We require $\rho(x|y)$ to be a density and satisfy $\rho(x|y) \in W^{m}_2([\calX] \times [\calY]).$
\begin{theorem}[Generalization bound]\label{th:generalizationbound} Assume that between-locality \eqref{eq:between-locality-stationarity1} holds and that $\rho$ and $\ell$ satisfy the assumptions above. If the regularization parameter of KRR \eqref{eq:alphakrr} is set to $\lambda = r(\frac{1}{m} + \frac{q}{|P|n})^{1/2}$, then the following holds:
\begin{equation}\label{eq:generalizationbound}
    \left|\Expect\calE(\xfunhat) - \calE(\xfun^\star_\sigma)\right| \leq C c_{B,\sigma}^{1/2}\left(\frac{1}{m} + \frac{1+q}{|P|n}\right)^{1/4},
\end{equation}
where $C\geq 0$ is a constant independent of the $n, m, P, B$ and $\sigma$.
\end{theorem}
The main difference with the analysis of \cite{ciliberto19localized} is the constant $c_{B,\sigma}$. In \cref{app:theoreticalanalysis} the theorem is proved under more general assumptions on $\rho, \ell$ relating their regularity to the reproducing kernel Hilbert space associated to $k$. The proof is based on the analysis by \cite{Ciliberto17multitask} and \cite{ciliberto19localized}.
The assumption on $\ell$ that we make here for the sake of readability is quite mild and it is satisfied by many loss functions including squared loss and Euclidean loss (see \cref{app:assumptionslosstarget}). 
The condition on the data distribution for \cref{th:generalizationbound} to hold are common for this kind of analysis \cite{caponnetto2007optimal} and in structured prediction settings \cite{ciliberto2016consistent,ciliberto19localized,ciliberto20general}. See in \cref{app:theoreticalanalysis} a more general statement with respect to the kernel $k$.

\textbf{Effect of patch correlation.} Theorem \ref{th:generalizationbound} states that the statistical performance of the estimator depends on $q$, i.e., the amount of correlation between patches. In particular, large correlation
translates into $q \approx |P|$, leading to a rate of $O(n^{-1/4})$, while small correlation improves the rate up to $O(m^{-1/4})$ (remember that $m \gg n$ in general) and up to $O((n|P|)^{-1/4})$ when taking~$m\propto n|P|$. This is particularly beneficial in problems with a large number of patches and allows to achieve low prediction error even with $n$ small, i.e., few training examples. The intuition is the following: when the patches from which we learn are independent, the estimators at the patch level have better performance than when they are highly correlated, as more information is present at learning time. This creates a trade-off between the size of the patches and statistical performance: for large patches the between-locality property is more likely to hold, but dependency is stronger due to overlapping. In our experiments, we take $n=4$, $|P|\approx 10^5$ ({\em e.g.} a 256 $\times$ 256 image contains $|P|=62001$ $8 \times 8$ overlapping patches) and we show that $q\ll|P|$ in \cref{app:constantq}. This justifies why selecting~$m=10^4$ patches is enough to obtain good performance.

\textbf{Effect of the fitting term.}
The ratio $c_{B,\sigma}\leq 1$  measures the statistical complexity of learning with the loss \eqref{eq:expectedrisk} under the constraint $\|y-Bx\|_2^2\leq \sigma^2$.
Indeed, if the diameter of the constraint set is small, then the estimator will require less samples to achieve small error.
If the constraint is removed, the constant becomes $\lim_{\sigma\to\infty} c_{B,\sigma} = 1$. Moreover, it can be explicitly bounded for multiple settings. For denoising it is $\min(\sigma/\operatorname{diam}(\calX),1)$, and so the magnitude of the noise appears as a multiplying factor in the bound \eqref{eq:generalizationbound}. For upsampling with a factor $k$ and noise $\sigma=0$ it is $k^{-1/2}$, and for inpainting with a proportion of $s$ missing pixels from the image and $\sigma=0$ is $s^{1/2}$ (\cref{app:comparisoninequality}).

\section{Algorithm} \label{sec:algorithm}

\begin{algorithm}[t]
 \caption{SDCA for solving \eqref{eq:updatezp} with Euclidean loss and NW estimator.}
\label{alg:sdcaeuclidean}
\begin{algorithmic}
\STATE $z_p^1 = x_p + 1/\beta\sum_{i=1}^m\mu_i^{(0)}$ \;
\FOR{$k = 1, \dots, K$}
 \STATE Pick $i$ at random in $\{1,...,m \}$ or perform gap sampling \;
 \STATE $b_i^{k} = z_p^{k} - \mu_i^{k}/\beta - \xp^{(i)}$ \;
\STATE $\mu_i^{k+1} = -b_i^{k}\min(\alpha_i(y_p)/\|b_i^{k}\|_2, \beta)$ \;
\STATE $z_p^{k+1} = z_p^{k} + (\mu_i^{k+1} - \mu_i^{k})/\beta$ \;
\ENDFOR
\end{algorithmic}
\end{algorithm}

We now present the algorithmic scheme to compute our estimator for any loss $\ell$ and detail the specific cases of the MSE and Euclidean loss. For the sake of presentation, we remove the factor $|P|^{-1}$ in~\eqref{eq:finalestimator} and re-define the parameter $\gamma$ as $\gamma|P|$.

\textbf{Least squares.}
When $\ell$ is the mean-squared error, the solution of \eqref{eq:finalestimator} can be obtained by solving
\begin{equation}\label{eq:msesystem}
   \textstyle{\Big( \gamma B^\top B + \sum_{p \in P} \sum_{i=1}^m \alpha_i(y_p) R_p^\top R_p \Big) x = \gamma B^\top y + \sum_{p \in P} \sum_{i=1}^m  \alpha_i(y_p) R_p^\top \xp^{(i)}},
\end{equation}
where $R_p$ is a matrix that extracts the $p$-th patch $x_p$ from an image $x$ and $R_p^\top$ replaces $x_p$ at its initial location in $x$. 
This can be done using conjugate gradient descent.
The problem is easy to solve because it is decomposable pixel-wise since the mean-squared error is separable, as already noted in \cref{sec:localpriors}.
This is computationally very efficient, but in general not entirely satisfactory from a modeling viewpoint, since each patch of the sharp image represents an element composed of pixels that are shared with other patches, and thus each pixel is strongly dependent on its neighborhood.

\textbf{Generic loss.} For a generic loss $\ell$, we use the half-quadratic splitting (HQS) algorithm \cite{geman95halfquadratic},  
decomposing the energy into simpler sub-problems.
We introduce $|P|$ variables $z_p$ and solve
\begin{equation}\label{eq:hqs}
    \min_{x, z_1,\dots,z_P}~\sum_{i=1}^m \sum_{p \in P} \alpha_i(y_p) \ell(z_p, \xp^{(i)}) + \frac{\gamma}{2} \| y - Bx \|_2^2 + \frac{\beta}{2} \| x_p - z_p\|_2^2.
\end{equation}
HQS outperforms in our experiments the alternating direction method of multipliers (ADMM) \cite{parikh14proximal}.
By choosing $x_p^{(0)}$ as $y_p$ for all $p$, the HQS algorithm performs at each iteration $t$
\begin{eqnarray}
    z_p^{(t+1)} & = & \textstyle{\argmin_{z_p} \sum_{i=1}^n \alpha_i(y_p) \ell(z_p^{(t)}, \xp^{(i)}) + \beta^{(t)}/2 \| z_p^{(t)} - x_p^{(t)}\|_2^2, \quad \forall p \in P,} \label{eq:updatezp}\\
    x^{(t+1)} & = & \textstyle{\argmin_{x} \sum_{p \in P} \beta^{(t)} \| z_p^{(t+1)} - x_p^{(t)}\|_2^2 + \gamma/2 \| y - Bx^{(t)} \|_2^2,} \label{eq:updatex}\\
    \beta^{(t+1)} & = & \delta \beta^{(t)},
\end{eqnarray}
where $\beta^{(0)}$ a positive scalar and $\delta>1$ such that $\beta^{(t)} = \delta^t \ \beta^{(0)}$ exponentially grows to $+\infty$.
The $x$ update \eqref{eq:updatex} is a least-squares problem analogous to \eqref{eq:msesystem}, and the $z_p$ updates \eqref{eq:updatezp} have the form of a weighted-by-$\alpha$ quadratically regularized empirical risk minimization problem.
As shown in \cref{app:convexity}, convexity of the objective is guaranteed if the coefficients are computed using KRR and $\ell$ is convex in its first argument, so the problem remains tractable. 
The update can be solved using algorithms for finite-sums such as full-gradient methods \cite{marteau19newton} or stochastic methods where the individual losses are not required to be all convex \cite{shalev16dualfree} (see \cref{app:generalalgorithms}).
We now present an efficient algorithm when $\ell$ is the Euclidean loss and the $\alpha$'s are computed using the NW estimator~\eqref{eq:alphanw}.

\textbf{Euclidean loss with Nadaraya-Watson estimator.} The individual losses in the finite sum of \eqref{eq:updatezp} are all convex when using the NW estimator, as the $\alpha$'s are all positive. In this case, we can use the stochastic dual coordinate ascent (SDCA) algorithm of \cite{shalev2013stochastic}. It is a stochastic dual algorithm such that at each iteration it selects an index $i\in[n]$ at random, and maximizes the dual objective w.r.t. to the dual variable $\mu_i$. For the Euclidean loss, this maximization has a closed form solution as it is a projection into an $\ell_2$ ball.
The method has linear convergence, is hyperparameter free, and the dual gap can be used as a termination criterion, yielding a fast and efficient solver for problem~\eqref{eq:updatezp}.
Moreover, we perform non-uniform sampling by selecting the index $i\in[n]$ according to the magnitude of the individual dual gaps (see \cref{app:sdca}). 
The algorithm is detailed in \cref{alg:sdcaeuclidean} and our code for SDCA with gap sampling is given in \cref{app:sdcaeuclidean}.

\section{Experiments}
\label{sec:experiments}
We experimentally validate the proposed algorithms for upsampling and deblurring, showing results on par with standard variational methods.
Our code will be made available and further experimental results can be found in \cref{app:experiments} and \cref{app:additionalimages}.

\textbf{Implementation details.} We consider the overlapping patches of the image as its set of parts $P$, as done in \cite{sun10contextconstrained, zoran11learning}.
We randomly select $m=10,000$, $8 \times 8$-patches from $n=4$ training images of the BSD300 dataset \cite{martin01database} to build the external dataset.
Increasing $m$ further improves the results and this value is thus a good compromise between speed and accuracy.
To compute $\alpha$, we use a Gaussian kernel comparing the discrete cosine transform (DCT) coefficients of the patches, as in \cite{dabov07image}.
When predicting $z_p$ with SDCA, the training patches $\xp^{(i)}$ and the current value of $x_p$ are centered for better reconstruction \cite{mairal09nonlocal, mairal14sparse, zoran11learning, sun10contextconstrained}.
The mean of $x_p$ is added to $z_p$ at the end of the SDCA loop.
We recompute $\alpha$ at each iteration $t$ using $x^{(t-1)}$, since it is a better estimate than~$y$.

\begin{figure}[t]
    \centering
    \adjustbox{max width=0.99\textwidth}{
    \begin{tabular}{cccc}
        \begin{subfigure}{0.25\textwidth}\includegraphics[scale=0.47]{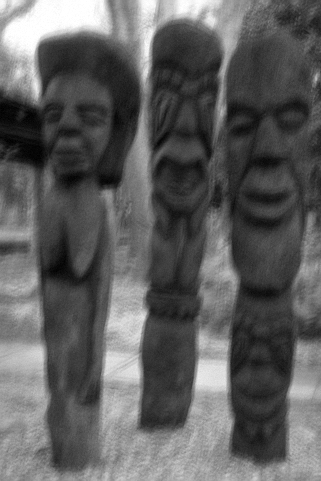}\caption*{Blurry image (20.57dB).}\end{subfigure} &  
        \begin{subfigure}{0.25\textwidth}\includegraphics[scale=0.34]{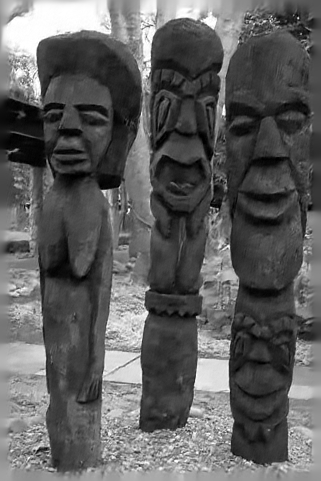}\caption*{EPLL (24.51dB).}\end{subfigure} & 
        \begin{subfigure}{0.25\textwidth}\includegraphics[scale=0.47]{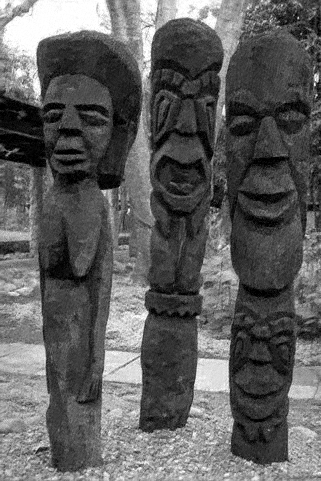}\caption*{Ours ($\ell_2$, 25.45dB).}\end{subfigure} &
        \begin{subfigure}{0.25\textwidth}\includegraphics[scale=0.47]{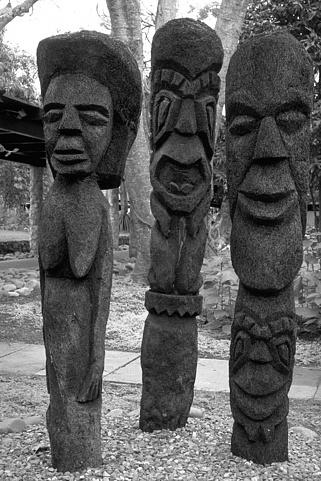}\caption*{Ground-truth image.}\end{subfigure} \\
    \end{tabular}
    }
    \caption{A non-blind deblurring example. Our method with the Euclidean loss achieves the best result in terms of PSNR (number in parenthesis) compared to EPLL.
    Visually, we also restore more details such as textures on the wooden totems whereas EPLL smooths these areas.}
    \label{fig:nbd}
\end{figure}

\textbf{Non-blind deblurring.}
In this setting, $B$ is the matrix form of a convolution with a blur kernel.
We use the experimental protocol of \cite{zoran11learning}: we blur the 68 test images of \cite{martin01database} with the two kernels of \cite{krishnan09hyperlaplacian} and add $1\%$ Gaussian noise.
We compare in \cref{tab:nbd} the proposed algorithms with the two variational methods of \cite{krishnan09hyperlaplacian} (handcrafted prior on image gradients) and \cite{zoran11learning} (GMM prior on patches).
Bold font indicates the best performer, and if marginally below, the second best too.
For the MSE-based method, we set $\gamma=5000$ and for the $\ell_2$-based solvers, we set $\gamma=3200$, $\beta_0 = 3$, $\delta = 2$ and $T=8$.
For the two kernels, the $\ell_2$-based solver's performance is on par with EPLL and outperforms HL by margins of about 1.5dB.
Note that the MSE version of our solver does not do as well as EPLL and its $\ell_2$- version, and achieves similar results to the gradient-based prior of HL.
It also shows that the choice of the loss, besides the kernel and computation of the $\alpha$, is an important aspect of our approach.
An example is shown in \cref{fig:nbd}.

\begin{table}[h]
\caption{Average PSNR for non-blind deblurring. The baselines are taken from \cite{zoran11learning}.}
    \centering
    \begin{tabular}{ccccc} \toprule
        Methods & HL \cite{krishnan09hyperlaplacian} & EPLL \cite{zoran11learning} & Ours~(MSE) & Ours~($\ell_2$) \\ \midrule
        Kernel 1 ($17 \times 17$) & 25.84 & \textbf{27.17} & 25.68 & \textbf{27.21} \\ 
        Kernel 2 ($19 \times 19$) & 26.38 & 27.70 & 26.31 & \textbf{27.96}  \\ \bottomrule
    \end{tabular}
    \label{tab:nbd}
\end{table}

\textbf{Upsampling.}
We follow \cite{sun10contextconstrained, farsiu04challengesSR} and implement $B$ as a convolution with an anti-aliasing Gaussian filter of standard deviation of 0.8, followed by decimation of factor 2. We evaluate our method in the test set Set5 used in \cite{zeyde10scaleup}, and compare it in \cref{tab:sisr} with bicubic interpolation and two variational methods \emph{specialized} for upsampling; KSVD \cite{zeyde10scaleup} and ANR \cite{timofte13anchored} (dictionary learning-based). For the $\ell_2$-based solver, we set $\gamma=6000$, $\beta_0 = 0.5$, $\delta=2$, $T=3$. We obtain consistent better results than bicubic interpolation (+1.7dB) and slightly worse (-0.5dB) than specialized up-sampling methods KSVD and ANR.
Note that we rely on simple DCT features and better results might be obtained with more adequate features. Due to a lack of space, qualitative results can be found in \cref{app:additionalimages}.

\begin{table}[h!]
\caption{PSNRs for upsampling.}
    \centering
    \begin{tabular}{lcccc}\toprule
        Methods & Bicubic & KSVD \cite{zeyde10scaleup} & ANR \cite{timofte13anchored} & Ours~($\ell_2$) \\ \midrule
        Baby & 37.07 & 38.25 & \textbf{38.44} & 38.29 \\ 
        Bird & 36.81 & 39.93 & \textbf{40.04} & 38.88 \\
        Butterfly & 27.43 & \textbf{30.65} & 30.48 & 29.97\\
        Head & 34.86 & 35.59 & \textbf{35.66} & \textbf{35.63}\\
        Woman & 32.14 & \textbf{34.49} & \textbf{34.55} & 34.03 \\ \midrule
        Average & 33.66 & \textbf{35.78} & \textbf{35.83} & 35.36\\ \bottomrule 
    \end{tabular}
    \label{tab:sisr}
\end{table}

\textbf{Other image restoration tasks.} Denoising and inpainting are other local image restoration tasks where our framework can be applied to. In the case of denoising, the problem is less structured as the operator $B$ is the identity, so our method highly relies on the performance of the learned regularizers. We perform below standard methods such as  \cite{elad06image, dabov07image, zoran11learning} in classical benchmarks (see \cref{app:experiments}). We think the main bottleneck is the data-driven regularizer: it is key to design kernels that are robust to large noise such as in \cite{takeda07kernel, chatterjee12patch}, and KRR may be needed instead of NW as the learning task is more important than the fitting term. We leave the task of fine-tuning our method for the specific problem of denoising for future work.

\section{Conclusion}

We have presented a new general framework for image restoration using localized structured prediction and non-linear multi-task learning that comes with statistical guarantees.
We have implemented two settings using the Euclidean distance or its square as a loss, and demonstrated the effectiveness of our algorithms on several problems, achieving results comparable or better than existing variational methods.
As a final note, we want to stress again the fact that our aim in this paper is {\em not} to obtain a new state of the art for the various image restoration tasks we have experimented with. Indeed, we are quite aware that some highly-tuned CNN architectures gave significantly higher performance for image denoising (e.g.,  \cite{zhang17beyond}), non-blind deblurring (e.g., \cite{zhang17fullyconvolutional}), and upsampling (e.g., \cite{dong16srcnn}). Our objective is instead to demonstrate that recent methods from the fields of structured prediction and non-linear multi-task learning with strong theoretical guarantees can effectively be adapted to image restoration problems, which we hope to have demonstrated with our experiments and comparisons with classical variational methods.

\section{Broader impact}
The main application of this paper is image restoration, with well known
benefits (from personal photography and medical imaging to astronomy and
microscopy) and risks (military and surveillance uses, and privacy issues related to deblurring images that were blurred to anonymize).
These concerns are ubiquitous in machine learning in general and its applications to computer vision and image processing in particular, and beyond the scope of this presentation.

\paragraph{Aknowledgments.}
This work was supported in part by the Inria/NYU collaboration and the Louis Vuitton/ENS chair on artificial intelligence. In addition, this work was funded in part by the French government under management of Agence Nationale de la Recherche as part of the ``Investissements d’avenir” program, reference ANR19-P3IA-0001 (PRAIRIE 3IA Institute). ANV received support from La Caixa Fellowship.

\bibliography{example_paper}

\begin{thebibliography}{10}

\bibitem{abdelhamed18sidd}
A.~Abdelhamed, S.~Lin, and M.~S. Brown.
\newblock A high-quality denoising dataset for smartphone cameras.
\newblock In {\em Proceedings of the conference on Computer Vision and Pattern
  Recognition}, pages 1692--1700, 2018.

\bibitem{adams2003sobolev}
R.~A. Adams and J.~J. Fournier.
\newblock {\em Sobolev spaces}.
\newblock Elsevier, 2003.

\bibitem{aronszajn1950theory}
N.~Aronszajn.
\newblock Theory of reproducing kernels.
\newblock {\em Transactions of the American mathematical society},
  68(3):337--404, 1950.

\bibitem{berlinet2011reproducing}
A.~Berlinet and C.~Thomas-Agnan.
\newblock {\em Reproducing kernel Hilbert spaces in probability and
  statistics}.
\newblock Springer Science \& Business Media, 2011.

\bibitem{buades05nonlocal}
A.~Buades, B.~Coll, and J.~Morel.
\newblock A non-local algorithm for image denoising.
\newblock In {\em Proceedings of the conference on Computer Vision and Pattern
  Recognition}, pages 60--65, 2005.

\bibitem{caponnetto2007optimal}
A.~Caponnetto and E.~De~Vito.
\newblock Optimal rates for the regularized least-squares algorithm.
\newblock {\em Foundations of Computational Mathematics}, 7(3):331--368, 2007.

\bibitem{chatterjee12patch}
P.~Chatterjee and P.~Milanfar.
\newblock Patch-based near-optimal image denoising.
\newblock {\em IEEE Transactions on Image Processing}, 21(4):1635--1649, 2012.

\bibitem{chen17trainable}
Y.~Chen and T.~Pock.
\newblock Trainable nonlinear reaction diffusion: {A} flexible framework for
  fast and effective image restoration.
\newblock {\em {IEEE} Transactions on Pattern Analysis and Machine
  Intelligence}, 39(6):1256--1272, 2017.

\bibitem{ciliberto19localized}
C.~Ciliberto, F.~Bach, and A.~Rudi.
\newblock Localized structured prediction.
\newblock In {\em Advances on Neural Information Processing Systems}, pages
  7299--7309, 2019.

\bibitem{ciliberto2016consistent}
C.~Ciliberto, L.~Rosasco, and A.~Rudi.
\newblock A consistent regularization approach for structured prediction.
\newblock In {\em Advances on Neural Information Processing Systems}, pages
  4412--4420, 2016.

\bibitem{ciliberto20general}
C.~Ciliberto, L.~Rosasco, and A.~Rudi.
\newblock A general framework for consistent structured prediction with
  implicit loss embeddings.
\newblock {\em Journal of Machine Learning Research}, page to appear, 2020.

\bibitem{Ciliberto17multitask}
C.~Ciliberto, A.~Rudi, L.~Rosasco, and M.~Pontil.
\newblock Consistent multitask learning with nonlinear output relations.
\newblock In {\em Advances on Neural Information Processing Systems}, pages
  1986--1996, 2017.

\bibitem{dabov07image}
K.~Dabov, A.~Foi, V.~Katkovnik, and K.~O. Egiazarian.
\newblock Image denoising by sparse {3-D} transform-domain collaborative
  filtering.
\newblock {\em IEEE Transactions on Image Processing}, 16(8):2080--2095, 2007.

\bibitem{dong16srcnn}
C.~Dong, C.~C. Loy, K.~He, and X.~Tang.
\newblock Image super-resolution using deep convolutional networks.
\newblock {\em {IEEE} Transactions on Pattern Analysis and Machine
  Intelligence}, 38(2):295--307, 2016.

\bibitem{elad06image}
M.~Elad and M.~Aharon.
\newblock Image denoising via sparse and redundant representations over learned
  dictionaries.
\newblock {\em {IEEE} Transactions on Image Processing}, 15(12):3736--3745,
  2006.

\bibitem{farsiu04challengesSR}
S.~Farsiu, M.~D. Robinson, M.~Elad, and P.~Milanfar.
\newblock Advances and challenges in super-resolution.
\newblock {\em International Journal on Imaging Systems and Technology},
  14(2):47--57, 2004.

\bibitem{freeman02example}
W.~T. Freeman, T.~R. Jones, and E.~C. Pasztor.
\newblock Example-based super-resolution.
\newblock {\em {IEEE} Computer Graphics and Applications}, 22(2):56--65, 2002.

\bibitem{geman95halfquadratic}
D.~Geman and C.~Yang.
\newblock Nonlinear image recovery with half-quadratic regularization.
\newblock {\em {IEEE} Transactions on Image Processing}, 5(7):932--946, 1995.

\bibitem{glasner09srsingleimage}
D.~Glasner, S.~Bagon, and M.~Irani.
\newblock Super-resolution from a single image.
\newblock In {\em Proceedings of the International Conference on Computer
  Vision}, pages 349--356, 2009.

\bibitem{he14inpainting}
K.~He and J.~Sun.
\newblock Image completion approaches using the statistics of similar patches.
\newblock {\em IEEE Transactions on Pattern Analysis and Machine Intelligence},
  36(12):2423--2435, 2014.

\bibitem{jacob09group}
L.~Jacob, G.~Obozinski, and J.~Vert.
\newblock Group lasso with overlap and graph lasso.
\newblock In {\em Proceedings of the International Conference on Machine
  Learning}, pages 433--440, 2009.

\bibitem{katkovnik10local}
V.~Katkovnik, A.~Foi, K.~O. Egiazarian, and J.~Astola.
\newblock From local kernel to nonlocal multiple-model image denoising.
\newblock {\em International Journal on Computer Vision}, 86(1):1--32, 2010.

\bibitem{krishnan09hyperlaplacian}
D.~Krishnan and R.~Fergus.
\newblock Fast image deconvolution using hyper-{Laplacian} priors.
\newblock In {\em Advances on Neural Information Processing Systems}, pages
  1033--1041, 2009.

\bibitem{le2018adaptive}
R.~Le~Priol, A.~Pich{\'e}, and S.~Lacoste-Julien.
\newblock Adaptive stochastic dual coordinate ascent for conditional random
  fields.
\newblock In {\em Proceedings of the Conference on Uncertainty in Artificial
  Intelligence ({UAI})}, pages 815--824, 2018.

\bibitem{mairal14sparse}
J.~Mairal, F.~Bach, and J.~Ponce.
\newblock Sparse modeling for image and vision processing.
\newblock {\em Foundations and Trends{\textregistered} in Computer Graphics and
  Vision}, 8(2-3):85--283, 2014.

\bibitem{mairal10online}
J.~Mairal, F.~Bach, J.~Ponce, and G.~Sapiro.
\newblock Online learning for matrix factorization and sparse coding.
\newblock {\em Journal of Machine Learning Research}, 11:19--60, 2010.

\bibitem{mairal09nonlocal}
J.~Mairal, F.~Bach, J.~Ponce, G.~Sapiro, and A.~Zisserman.
\newblock Non-local sparse models for image restoration.
\newblock In {\em Proceedings of the International Conference on Computer
  Vision}, pages 2272--2279, 2009.

\bibitem{marteau19newton}
U.~Marteau{-}Ferey, F.~Bach, and A.~Rudi.
\newblock Globally convergent newton methods for ill-conditioned generalized
  self-concordant losses.
\newblock In {\em Advances in Neural Information Processing Systems}, pages
  7634--7644, 2019.

\bibitem{martin01database}
D.~R. Martin, C.~C. Fowlkes, D.~Tal, and J.~Malik.
\newblock A database of human segmented natural images and its application to
  evaluating segmentation algorithms and measuring ecological statistics.
\newblock In {\em Proceedings of the International Conference on Computer
  Vision}, pages 416--425, 2001.

\bibitem{nadaraya64estimating}
E.~A. Nadaraya.
\newblock On estimating regression.
\newblock {\em Teor. Veroyatnost. i Primenen.}, 9(1):157--159, 1964.

\bibitem{osokin2016minding}
A.~Osokin, J.~Alayrac, I.~Lukasewitz, P.~K. Dokania, and S.~Lacoste{-}Julien.
\newblock Minding the gaps for block frank-wolfe optimization of structured
  svms.
\newblock In {\em Proceedings of the International Conference on Machine
  Learning ({ICML})}, pages 593--602, 2016.

\bibitem{parikh14proximal}
N.~Parikh and S.~P. Boyd.
\newblock Proximal algorithms.
\newblock {\em Foundations and Trends in Optimization}, 1(3):127--239, 2014.

\bibitem{parzen1962estimation}
E.~Parzen.
\newblock On estimation of a probability density function and mode.
\newblock {\em The annals of mathematical statistics}, 33(3):1065--1076, 1962.

\bibitem{rockafellar1970convex}
R.~T. Rockafellar.
\newblock {\em Convex analysis}.
\newblock Number~28. Princeton university press, 1970.

\bibitem{rudin92totalvariation}
L.~I. Rudin, S.~Osher, and E.~Fatemi.
\newblock Nonlinear total variation based noise removal algorithms.
\newblock {\em Physica D}, 60:259--268, 1992.

\bibitem{scholkopf2002learning}
B.~Sch{\"o}lkopf, A.~J. Smola, F.~Bach, et~al.
\newblock {\em Learning with kernels: support vector machines, regularization,
  optimization, and beyond}.
\newblock MIT press, 2002.

\bibitem{shalev16dualfree}
S.~Shalev{-}Shwartz.
\newblock {SDCA} without duality, regularization, and individual convexity.
\newblock In {\em Proceedings of the International Conference on Machine
  Learning}, pages 747--754, 2016.

\bibitem{shalev2016sdca}
S.~Shalev-Shwartz.
\newblock Sdca without duality, regularization, and individual convexity.
\newblock In {\em International Conference on Machine Learning}, pages
  747--754, 2016.

\bibitem{shalev2013stochastic}
S.~Shalev-Shwartz and T.~Zhang.
\newblock Stochastic dual coordinate ascent methods for regularized loss
  minimization.
\newblock {\em Journal of Machine Learning Research}, 14(Feb):567--599, 2013.

\bibitem{starck06astronomical}
J.~Starck and F.~Murtagh.
\newblock {\em Astronomical Image and Data Analysis, Second Edition}.
\newblock Astronomy and Astrophysics Library. Springer, 2006.

\bibitem{sun10contextconstrained}
J.~Sun, J.~Zhu, and M.~F. Tappen.
\newblock Context-constrained hallucination for image super-resolution.
\newblock In {\em Proceedings of the conference on Computer Vision and Pattern
  Recognition}, pages 231--238, 2010.

\bibitem{takeda07kernel}
H.~Takeda, S.~Farsiu, and P.~Milanfar.
\newblock Kernel regression for image processing and reconstruction.
\newblock {\em {IEEE} Transactions on Image Processing}, 16(2):349--366, 2007.

\bibitem{timofte13anchored}
R.~Timofte, V.~D. Smet, and L.~V. Gool.
\newblock Anchored neighborhood regression for fast example-based
  super-resolution.
\newblock In {\em Proceedings of the International Conference on Computer
  Vision}, pages 1920--1927, 2013.

\bibitem{vapnik2013nature}
V.~Vapnik.
\newblock {\em The nature of statistical learning theory}.
\newblock Springer science \& business media, 2013.

\bibitem{vershynin2018high}
R.~Vershynin.
\newblock {\em High-dimensional probability: An introduction with applications
  in data science}, volume~47.
\newblock Cambridge university press, 2018.

\bibitem{wang08alternating}
Y.~Wang, J.~Yang, W.~Yin, and Y.~Zhang.
\newblock A new alternating minimization algorithm for total variation image
  reconstruction.
\newblock {\em {SIAM} Journal on Imaging Sciences}, 1(3):248--272, 2008.

\bibitem{weigert18microscopy}
M.~Weigert, U.~Schmidt, T.~Boothe, A.~M\"uller, A.~Dibrov, A.~Jain, B.~Wilhelm,
  D.~Schmidt, C.~Broaddus, S.~Culley, M.~Rocha-Martins, F.~Segovia-Miranda,
  C.~Norden, R.~Henriques, M.~Zerial, M.~Solimena, J.~Rink, P.~Tomancak,
  L.~Royer, F.~Jug, and E.~W. Myers.
\newblock Content-aware image restoration: Pushing the limits of fluorescence
  microscopy.
\newblock {\em Nature Methods}, 15(12):1090--1097, 2018.

\bibitem{wiener49extrapolation}
N.~Wiener.
\newblock {\em Extrapolation, Interpolation, and Smoothing of Stationary Time
  Series}.
\newblock Wiley, 1949.

\bibitem{zeyde10scaleup}
R.~Zeyde, M.~Elad, and M.~Protter.
\newblock On single image scale-up using sparse-representations.
\newblock In {\em Proceedings of the International Conference on Curves and
  Surfaces}, pages 711--730, 2010.

\bibitem{zhang17fullyconvolutional}
J.~Zhang, J.~Pan, W.~Lai, R.~W.~H. Lau, and M.~Yang.
\newblock Learning fully convolutional networks for iterative non-blind
  deconvolution.
\newblock In {\em Proceedings of the conference on Computer Vision and Pattern
  Recognition}, pages 6969--6977, 2017.

\bibitem{zhang17beyond}
K.~Zhang, W.~Zuo, Y.~Chen, D.~Meng, and L.~Zhang.
\newblock Beyond a gaussian denoiser: Residual learning of deep {CNN} for image
  denoising.
\newblock {\em {IEEE} Transactions on Image Processessing}, 26(7):3142--3155,
  2017.

\bibitem{zoran11learning}
D.~Zoran and Y.~Weiss.
\newblock From learning models of natural image patches to whole image
  restoration.
\newblock In {\em Proceedings of the International Conference on Computer
  Vision}, pages 479--486, 2011.

\end{thebibliography}
\bibliographystyle{abbrv}

\newpage
\appendix

{\Huge{Organization of the Appendix}}

\begin{itemize}
    \item [\large{\textbf{A.}}] \textbf{\large{Theoretical Analysis of the Estimator}} \vspace{0.2cm} \\
    Analysis of \cref{sec:stats}.
    \begin{itemize}
        \item [\textbf{A.1.}] {\bfseries Derivation of the Estimator} 
        \item [\textbf{A.2.}] {\bfseries Comparison Inequality} 
        \item [\textbf{A.3.}] {\bfseries Generalization Bound} 
        \item [\textbf{A.4.}] {\bfseries Assumptions on the Loss $\boldsymbol{\ell}$ and Target Distribution}
    \end{itemize}
    \vspace{0.1cm} 
    \item [\large{\textbf{B.}}] \textbf{\large{Algorithm}} \vspace{0.2cm} \\
    Details of \cref{sec:algorithm}.
    \begin{itemize}
        \item [\textbf{B.1.}] {\bfseries Convexity of the Energy} 
        \item [\textbf{B.2.}] {\bfseries General Algorithms to Solve the $\boldsymbol{z_p}$ Update} 
        \item [\textbf{B.3.}] {\bfseries SDCA for Euclidean Loss and Nadaraya-Watson Estimator} 
    \end{itemize}
    \vspace{0.1cm} 
    \item [\large{\textbf{C.}}] \textbf{\large{Experiments}}
    \vspace{0.2cm} \\
    Details of \cref{sec:experiments}.
    \begin{itemize}
        \item [\textbf{C.1.}] {\bfseries Analysis of the Constant $\boldsymbol{q}$} 
        \item [\textbf{C.2.}] {\bfseries Implementation Details}
        \item [\textbf{C.3.}] {\bfseries Further Experiments}
    \end{itemize}
    \item [\large{\textbf{D.}}] \textbf{\large{Additional Images}}
    \vspace{0.2cm} \\
    More visualizations.
    \vspace{0.1cm} 
\end{itemize}

\section{Theoretical Analysis of the Estimator} \label{app:theoreticalanalysis}


\subsection{Derivation of the Estimator} \label{app:derivationestimator}
Recall the goal is to estimate a function $\xfun:\calY\xrightarrow{}\calX$ that minimizes the expected risk
\begin{equation}\label{eq:riskproblem}
   \calE(\xfun) = 
  \Expect_{(y, x)\sim\rho}~ L(\xfun(y),x),
\end{equation}
under the constraint $\|y- B \xfun(y)\|_2^2\leq\sigma^2$, where the loss function $L$ decomposes by patches as
\begin{equation}
    L(x,x') =  \frac{1}{|P|}\sum_{p\in P}\Lp(x_p,x_p').
\end{equation}
In this subsection, we derive the proposed estimator of \cref{eq:finalestimator} from the risk minimization problem~\eqref{eq:expectedrisk}.
More specifically, in \cref{eq:characterizationoptimum} we compute the form of the exact minimizer $\xfun_{\sigma}^\star$, in \cref{app:decompositionloss}, we describe the assumption on the loss function $\Lp$ needed for the analysis, and in \cref{app:quadraticsurrogate} we use this decomposition to construct a quadratic surrogate, which provides the resulting estimator when minimized using kernel ridge regression.

\subsubsection{Characterization of the optimum}\label{eq:characterizationoptimum} 

\begin{lemma}\label{lemma:regressionfunction}
The minimizer $\xfun_{\sigma}^\star$ of the risk \eqref{eq:riskproblem} under the constraint $\|y- B \xfunhat(y)\|_2^2\leq\sigma^2$ takes the form
\begin{equation}\label{eq:regressionfunction}
   \xfun_{\sigma}^\star(y) \in  \argmin_{\|y-Bx\|_2^2\leq\sigma^2} \Omega_{L}^{\pi}(x~|~y), \quad \text{where} \quad \Omega_L^{\pi}(x~|~y) = \Expect_{x'\sim\pi_y} L(x,x'),
\end{equation}
where $\pi_y(x) = \rho(x~|~y)$ denotes the conditional distribution of $x$ given $y$.
\end{lemma}
\begin{proof}
We first show that the problem can be solved independently for every $y\in\calY$.

\begin{equation}
   \xfun_{\sigma}^\star =
    \underset{\substack{\xfun:\calY\xrightarrow{}\calX \\ \|y-B \xfun(y)\|_2^2\leq\sigma^2}}{\operatorname{argmin}}\Expect_{(y, x)\sim\rho} L(\xfun(y),x)
    =\underset{\substack{\xfun:\calY\xrightarrow{}\calX \\ \|y-B\xfun(y)\|_2^2\leq\sigma^2}}{\operatorname{argmin}}\Expect_{x\sim\rho_{\calX}} \Expect_{x\sim\pi_y} L(\xfun(y),x),
\end{equation}

where $\rho_{\calX}$ denotes the marginal distribution in $\calX$.
Hence, the problem decouples in $y$ and one can write 
\begin{equation*}
    \xfun_{\sigma}^\star(y) = \argmin_{\substack{x\in\calX \\ \|y-Bx\|_2^2\leq\sigma^2}}~\Expect_{x'\sim\pi_y} L(x,x').
\end{equation*}
See Lemma~6 of \cite{ciliberto2016consistent} for the full proof (taking into account also measure theoretic aspects).
\end{proof}

For the rest of the Appendix, we assume that \emph{between-locality} (introduced in \cref{sec:localpriors}) holds. We recall here the assumption:
\begin{itemize}
    \item The sharp patch $x_p$ is conditionally independent from $y$ given $y_p$.
    \item The following holds for all patches $p,p'\in P$,
    \begin{equation}\label{eq:between-locality-stationarity}
       \pi_{y_p}(x_p) = \pi_{y_{p^\prime}}(x_{p^\prime})
    \end{equation}
\end{itemize}
In particular, this allows us to write 
\begin{equation}
    \Omega_L^{\pi}(x~|~y) =  \frac{1}{|P|}\sum_{p \in P} \Omega_{\ell}^{\pi}(x_p~|~ y_p),\quad \textrm{where}, \quad \Omega_{\ell}^{\pi}(\xp~|~\yp) = \Expect_{\xp'\sim\pi_{\yp}}\ell(\xp,\xp').
\end{equation}
\subsubsection{Decomposition of the loss}\label{app:decompositionloss}
In order to derive the estimator and to proceed further with the theoretical analysis, we need to introduce a mild assumption on the loss $\Lp$ at the patch level. 

\textbf{SELF assumption of the loss $\boldsymbol{\ell}$.} There exists a separable Hilbert space $\calH$ and continuous bounded maps $\psi, \phi:[\calX]\xrightarrow{}\calH$ such that the loss function $\Lp$ between sharp image patches decomposes as:
\begin{equation}\label{eq:SELFapp}
   \Lp(\xp,\xp') = \langle\psi(\xp), \phi(\xp')\rangle_{\calH},
\end{equation}
for all $\xp,\xp'\in[\calX]$.

In particular, it was shown by \cite{ciliberto2016consistent} (see Thm. 19) that the above assumption holds if the loss function $\Lp$ is absolutely continuous and $[\calX]$ is a compact space, which is our case. Moreover, note that the assumption is always satisfied when the spaces are discrete, as \cref{eq:SELFapp} corresponds to a low-rank decomposition of the loss matrices. In \cref{app:assumptionslosstarget} we study the SELF assumption more explicitly under the hypothesis of the loss from our main \cref{th:generalizationbound}.

\subsubsection{Construction of the quadratic surrogate}\label{app:quadraticsurrogate}

\paragraph{Notation. }
In this subsection we are going to use a construction widely used in the context of quadratic surrogate approaches for structured prediction, with the goal of deriving the estimator in \cref{eq:finalestimator} (for more details on this kind of construction see \cite{ciliberto20general}).
The assumption \eqref{eq:SELFapp} on the loss functions allows us to write the exact minimizer $\xfun_{\sigma}^\star$ with respect to the functions $\gstar:[\calY]\xrightarrow{}\calH$ defined as the conditional expectation of the embedding $\phi$:
\begin{equation}
    \gstar(\yp) = \Expect_{\xp\sim\pi_{\yp}}\phi(\xp).
\end{equation}
This can be seen by simply moving the conditional expectation inside the scalar product as:
\begin{align*}
    \xfun_{\sigma}^\star(y)  &\in  \argmin_{\substack{x\in\calX \\ \|y-B x\|_2^2\leq\sigma^2}}
  ~\Omega_L^{\pi}(x~|~y)\\
    &=  \argmin_{\substack{x\in\calX \\ \|y-B x\|_2^2\leq\sigma^2}}
  ~ \frac{1}{|P|}\sum_{p\in P}\Omega_{\ell}^{\pi}(x_p~|~y_p)\\
  &= \argmin_{\substack{x\in\calX \\ \|y-B x\|_2^2\leq\sigma^2}}
  ~ \frac{1}{|P|}\sum_{p\in P}\Expect_{x_p'\sim\pi_{y_p}}\langle\psi(x_p), \phi(x_p')\rangle_{\calH} \\
  &= \argmin_{\substack{x\in\calX \\ \|y-B x\|_2^2\leq\sigma^2}} 
  ~ \frac{1}{|P|}\sum_{p\in P}\langle\psi(x_p), \gstar(y_p)\rangle_{\calH} \\
\end{align*}
A natural strategy for building an approximator $\xfunhat$ of $\xfun_{\sigma}^\star$ is to design an estimator $\ghat$ of $\gstar$ and consider:

\begin{equation}\label{eq:estimatorconstrained}
    \xfunhat(y) \in \argmin_{\|y-B x\|_2^2\leq\sigma^2}
  ~ \frac{1}{|P|}\sum_{p\in P}\langle\psi(x_p), \ghat(y_p)\rangle_{\calH}.
\end{equation}
The question now boils down to constructing the estimator $\ghat$. The important observation is that as $\gstar$ is written as a conditional expectation, it is characterized as the minimum of the expected squared error measured in the Hilbert space $\calH$:

\begin{equation}\label{eq:expectedquadraticrisk}
    \gstar = \argmin_{\gfun:[\calY]\xrightarrow{}\calH}~\Expect_{(\yp,\xp)\sim\rho}\|\phi(\xp) - g(\yp)\|_{\calH}^2,
\end{equation}
where now $\rho$ denotes the distribution over patches instead of full images.
Given a dataset of patches~$(\yp^{(i)}, \xp^{(i)})_{1\leq i\leq m}$, we can approximate the minimizers of \cref{eq:expectedquadraticrisk} with kernel ridge regression as:
\begin{equation}\label{eq:krr}
    \ghat = \argmin_{\gfun\in\calG\otimes\calH}\frac{1}{m}\sum_{j=1}^m\|\phi(\xp^{(i)}) - g(\yp^{(i)})\|_{\calH}^2 + \lambda\|\gfun\|_{\calG\otimes\calH}^2,
\end{equation}
where $\calG$ is the scalar reproducing kernel Hilbert space (RKHS) associated to the positive definite kernel $k:[\calY]\times[\calY]\xrightarrow{}\Rspace{}$,  $\calG\otimes\calH$ denotes the vector-valued RKHS corresponding to the tensor product between $\calG$ and $\calH$, and $\lambda>0$ is a regularization parameter. The solution of \cref{eq:krr} can be computed in closed form as:
\begin{equation}\label{eq:krrsolution}
    \ghat(\yp) = \sum_{i=1}^m\alpha_{i}(\yp)\phi(\xp^{(i)}),
\end{equation}
where $\alpha(\yp)$ is defined as:
\begin{equation}\label{eq:weights_}
\alpha(\yp) = (K + m\lambda I)^{-1}v(\yp)\in \R^{m},
\end{equation}
where $v(\yp)$ is a vector in $\R^m$ with $i$-th component $k(\yp,\yp^{(i)})$ and
$K$ is a matrix in $\R^{m\times m}$ defined by $K(i,j)~=~k(\yp^{(i)}, \yp^{(j)})$.

Using the linearity of the scalar product and the fact that $\ghat$ is linear in the embeddings $\phi$ evaluated at training data points of sharp image patches, we obtain that the estimators $\widehat{\Omega}_{\ell}^{\pi}$ are independent of the embeddings \eqref{eq:SELFapp} of the loss function $\Lp$:
\begin{equation}
    \widehat{\Omega}_{\ell}^{\pi}(x_p~|~y_p) = \langle\psi(x_p), \widehat{g}(y_p)\rangle_{\calH} 
 = \sum_{i=1}^m\alpha_{i}(y_p)\langle\psi(x_p), \phi(\yp^{(i)})\rangle_{\calH} = \sum_{i=1}^m\alpha_{i}(y_p)\ell(x_p, \xp^{(i)}).
\end{equation}


\subsection{Comparison Inequality} \label{app:comparisoninequality}
In \cref{app:derivationestimator}, we have derived $\xfunhat$ by estimating a vector-valued function $\gstar$ taking values in a Hilbert space $\calH$ defined by the decomposition of the loss functions in \cref{eq:SELFapp}. The goal of this section, is to
analyze how the error of estimating $\gstar$ by $\ghat$ translates to the excess risk $\calE(\xfunhat) - \calE(\xfun_{\sigma}^\star)$, which is the quantity that we ultimately want to bound. This quantification is made precise by the following~\cref{th:comparisoninequality}, which is analogous to Thm. 9 by \cite{ciliberto19localized} but with a more careful analysis of the constants in order to make appear the effect of the constraint $\|y-B\xfun(y)\|_2^2\leq\sigma^2$.

\paragraph{Notation.} Let's first define $G(y) = g(y_p)_{p\in P}$ and  $\Psi(x) = \psi(x_p)_{p\in P}$, so that
\begin{equation}
   \langle\Psi(x), G(y)\rangle_{\calH_P} =  \frac{1}{|P|}\sum_{p\in P}\langle\psi(x_p), g(y_p)\rangle_{\calH},
\end{equation}
where $\calH_P = \bigoplus_{p\in P}\calH$ is the direct sum of $|P|$ copies of $\calH$.
t the last step, we used the fact that $\|y- B\widehat x(y)\|_2^2\leq\sigma^2$ and

\begin{theorem}[Comparison inequality]\label{th:comparisoninequality}
Assume that $\ell$ satisfies \eqref{eq:SELFapp}. Let $\xfunhat$ and $\xfun_{\sigma}^\star$ be defined in \cref{eq:estimatorconstrained} and \cref{eq:regressionfunction}, respectively. Then, 
\begin{equation}\label{eq:comparisoninequality}
    \calE(\xfunhat) - \calE(\xfun_{\sigma}^\star) \leq 2c_{B, \sigma}^{1/2}\big(\Expect_{y\sim\rho_{\calY}}\|\Ghat(y) - \Gstar(y)\|_{\calH_P}^2\big)^{1/2},
\end{equation}
where $\rho_{\calY}$ denotes the marginal distribution in $\calY$ and
\begin{equation}\label{eq:nonlinearmtlconstant}
        c_{B, \sigma} = \Expect_{y\sim\rho_{\calY}}\operatorname{diam}^2(\Psi(\calC_{B,\sigma}(y))),
\end{equation}
with $\calC_{B,\sigma}(y) = \{x\in\calX~|~\|Bx-y\|_2\leq\sigma\}$.
\end{theorem}
\begin{proof}
We have that
\begin{align*}
    \calE(\xfunhat) & - \calE(\xfun_{\sigma}^\star) \leq \Expect_{y\sim\rho_{\calY}} \langle\Psi(\xfunhat(y)), G^\star(y)\rangle - 
    \langle\Psi(\xfun_{\sigma}^\star(y)), G^\star(y)\rangle \\
    &= \Expect_{y\sim\rho_{\calY}} \langle\Psi(\xfunhat(y)), G^\star(y)-\widehat{G}(y)\rangle +  \underbrace{\langle\Psi(\xfunhat(y)) - \Psi(\xfun_{\sigma}^\star(y),\widehat{G}(y)\rangle}_{\leq 0} \\ 
    &+ \Expect_{y\sim\rho_{\calY}} \langle\Psi(\xfun_{\sigma}^\star(y)),  \widehat{G}(y) - G^\star(y)\rangle \\
    &= \Expect_{y\sim\rho_{\calY}}\langle\Psi(\xfunhat(y)), G^\star(y) -\widehat{G}(y)\rangle + \langle\Psi(\xfun_{\sigma}^\star(y)),  \widehat{G}(y) - G^\star(y)\rangle  \\
    &= \Expect_{y\sim\rho_{\calY}}\min_{v\in\calH_P}\big(\langle\Psi(\xfunhat(y))- v, G^\star(y) -\widehat{G}(y)\rangle + \langle\Psi(\xfun_{\sigma}^\star(y)) - v,  \widehat{G}(y) - G^\star(y)\rangle\big)\\
&\leq \Expect_{y\sim\rho_{\calY}}\sup_{x,x'\in\calC_{B,\sigma}(y)}\min_{v\in\calH_P}\big(\langle\Psi(\xfunhat(y))- v, G^\star(y) -\widehat{G}(y)\rangle + \langle\Psi(\xfun_{\sigma}^\star(y)) - v,  \widehat{G}(y) - G^\star(y)\rangle\big)\\
&\leq \Expect_{y\sim\rho_{\calY}}\sup_{x,x'\in\calC_{B,\sigma}(y)}\min_{v\in\calH_P} (\|\Psi(x)- v\|_{\calH_P} + \|\Psi(x')- v\|_{\calH_P}) \|G^\star(y) -\widehat{G}(y)\|\\
&\leq \Expect_{y\sim\rho_{\calY}}\sup_{x,x'\in\calC_{B,\sigma}(y)}\|\Psi(x)- \Psi(x')\|_{\calH_P}\|G^\star(y) -\widehat{G}(y)\|_{\calH_P}\\
&= \Expect_{y\sim\rho_{\calY}}\operatorname{diam}(\Psi(\calC_{B,\sigma}(y)))\|G^\star(y) -\widehat{G}(y)\|_{\calH_P}\\
&\leq \sqrt{\Expect_{y\sim\rho_{\calY}}\operatorname{diam}^2(\Psi(\calC_{B,\sigma}(y))) ~\times ~\Expect_{y\sim\rho_{\calY}}\|G^\star(y) -\widehat{G}(y)\|^2},
\end{align*}
where we have used that $\min_{v\in\calH}~\|u-v\|_{\calH}+\|u'-v\|_{\calH} = \|u-u'\|_{\calH}$ if $\calH$ is a Hilbert space.
\end{proof}

\begin{theorem}[Constant $c_{B,\sigma}$] Under the hypothesis on $\ell$ from \cref{th:generalizationbound}, the constant $c_{B,\sigma}$ is bounded as
\begin{equation}
    c_{B,\sigma} \leq \operatorname{diam}([\calX])\frac{\Expect_{y\sim\rho_{\calY}}\operatorname{diam}(\calC_{B,\sigma}(y))}{\operatorname{diam}(\calX)}.
\end{equation}
\end{theorem}
\begin{proof}
Let $c_{B,\sigma} = \Expect_{y\sim\rho_{\calY}}c_{B,\sigma}(y)$, where \begin{equation}
c_{B,\sigma}(y) = \operatorname{diam}^2(\Psi(\calC_{B,\sigma}(y))) = \sup_{x,x'\in\calC_{B,\sigma}(y)}\big\|\Psi(x)- \Psi(x')\big\|_{\calH_P}^2.    
\end{equation}
Using the hypothesis on $\ell$ from \cref{th:generalizationbound} and \cref{lm:cond-on-ell}, we can write $\langle\psi(x_p), \psi(x_p')\rangle_{\calH} = e^{-\|x_p-x_p'\|_2}$. Hence, the squared distance between embeddings $\psi$ takes the form $\|\psi(x_p) - \psi(x_p')\|_{\calH}^2= 1-e^{-\|x_p-x_p'\|_2}$. We have that~$c_{B,\sigma}^2 = \Expect_{y\sim\rho_{\calY}}c_{B,\sigma}(y)^2$, where 
\begin{align}
    c_{B,\sigma}(y) &= \sup_{x,x'\in\calC_{B,\sigma}(y)}\big\|\Psi(x)- \Psi(x')\big\|_{\calH_P}^2 \\
    &= \sup_{x,x'\in\calC_{B,\sigma}(y)} \frac{1}{|P|}\sum_{p\in P} \|\psi(x_p) - \psi(x_p')\|_{\calH}^2 \\
    &=\sup_{x,x'\in\calC_{B,\sigma}(y)} \frac{1}{|P|}\sum_{p\in P} 1-e^{-\|x_p-x_p'\|_2} \\
    &\leq \sup_{x,x'\in\calC_{B,\sigma}(y)} \frac{1}{|P|}\sum_{p\in P}\|x_p-x_p'\|_2 \\
    &\leq \Big(\sup_{x,x'\in\calC_{B,\sigma}(y)} \frac{1}{|P|}\sum_{p\in P} \|x_p-x_p'\|_2^2\Big)^{1/2} \\
    &= \frac{C^{1/2}}{|P|^{1/2}}\Big(\sup_{x,x'\in\calC_{B,\sigma}(y)}\|x-x'\|_2^2\Big)^{1/2} \\
    &= \frac{C^{1/2}}{|P|^{1/2}}\sup_{x,x'\in\calC_{B,\sigma}(y)}\|x-x'\|_2 \\
    &= \frac{C^{1/2}}{|P|^{1/2}}\operatorname{diam}(\calC_{B,\sigma}(y)),
\end{align}
where we have used that $\sum_{p\in P} \|x_p-x_p'\|_2^2 = C\|x-x'\|_2^2$, where $C$ is the number of patches a pixel belongs to.
Using that $\operatorname{diam}([\calX])^2 = \frac{C}{|P|}\operatorname{diam}(\calX)^2$, we obtain
\begin{equation}
    c_{B,\sigma}(y) \leq \operatorname{diam}([\calX])\frac{\operatorname{diam}(\calC_{B,\sigma}(y))}{\operatorname{diam}(\calX)}.
\end{equation}
Taking the expectation w.r.t $y$ in both sides we obtain the desired result.
\end{proof}

In the following \cref{prop:specificexamples}, we compute explicitly an upper bound of the constant $c_{B,\sigma}$ for the tasks of denoising, inpainting and downsampling.

\begin{proposition}[Bounding $c_{B,\sigma}$ for specific problems]\label{prop:specificexamples} The constant $c_{B,\sigma}$ can bounded explicitly for the following settings.
\begin{itemize}
    \item \textit{Denoising.} The degradation model is $y=x+\varepsilon$. Hence, $B=Id$ and $\sigma>0$. The constant takes the following form
    \begin{equation}
        c_{B,\sigma} \propto \min\left(\frac{\sigma}{\operatorname{diam}(\calX)}, 1\right)
    \end{equation}
    \item \textit{Inpainting.} The operator $B$ is a the identity matrix with a ratio of $s$ zeros in the diagonal. If we assume there is no noise, the constant takes the following form
    \begin{equation}
        c_{B,\sigma} \propto s^{1/2}.
    \end{equation}
    \item \textit{Downsampling. } The operator $B$ is the downsampling operator by a factor of $k$. If we assume there is no noise, the constant takes the following form
    \begin{equation}
        c_{B,\sigma} \propto k^{-1/2}.
    \end{equation}
\end{itemize}
\end{proposition}
\begin{proof}
We proceed separately for every setting. Note that the space of sharp images is the cube~$\calX=[0,1]^d$ where $d$ is the total number of pixels.
\begin{itemize}
    \item[-] \textit{Denoising.} The set $\calC_{B,\sigma}(y)$ is an $\ell_2$ ball centered at $y$ with radius $\sigma$ intersected with the cube $\calX$. Hence, it directly follows $\operatorname{diam}(\calC_{B,\sigma}(y)) \leq \sigma$.
     \item[-] \textit{Inpainting.} The set $\calC_{B,\sigma}(y)$ is a cube centered at $y$ of dimension $ds$. Using that the diameter of a cube of dimension $d$ is $\sqrt{d}$, we obtain that $\operatorname{diam}(\calC_{B,\sigma}(y))/\operatorname{diam}(\calX) \leq \sqrt{\frac{ds}{d}} = s^{1/2}$.
     \item[-] \textit{Downsampling. } The set $\calC_{B,\sigma}(y)$ is a cube centered at $y$ of dimension $d/k$, using the expression of the diameter of the cube as before, we obtain $\operatorname{diam}(\calC_{B,\sigma}(y))/\operatorname{diam}(\calX) \leq \sqrt{\frac{d/k}{d}} = \sqrt{1/k}$.
\end{itemize}
\end{proof}


\subsection{Generalization Bound} \label{app:generalizationbound}
The goal of this section is to prove \cref{th:generalizationbound} by first bounding the estimation error of the surrogate problem $\Expect_{y\sim\rho_{\calY}}\|\Gstar(y) - \Ghat(y)\|_{\calH_P}^2$ and applying the comparison inequality derived in \cref{th:comparisoninequality}. 
Define the quantities $\mathsf{g}, r, q$ as:
\begin{equation}
    \mathsf{g} = \|\gstar\|_{\calG\otimes\calH}, \hspace{0.5cm}
    r = \sup_{\yp\in[\calY]}k(\yp, \yp),
\end{equation}
\begin{equation}\label{eq:constantqapp}
 q =\frac{1}{|P|r^2}\sum_{p,p'\in P}\Expect_{y, y'} C_{p,p'}(y,y'), 
 \hspace{0.5cm} C_{p,p'}(y,y') = k(y_p,y_{p'})^2 - k(y_p,y_{p'}')^2.
\end{equation}
We have the following theorem.
\begin{theorem}[Generalization bound]\label{th:generalizationboundcomplete}
Assume that the decomposition of the loss $\ell$ in \cref{eq:SELFapp} holds and $\mathsf{g}<+\infty$.
Moreover, assume between-locality and let the regularization parameter of KRR be
\begin{equation}
    \lambda = r\sqrt{\frac{1}{m} + \frac{q}{|P|n}}.
\end{equation}
Then, we have that:
\begin{equation}\label{eq:generalbound}
    \Expect\calE(\xfunhat) - \calE(\xfun_{\sigma}^\star) \leq 12c_{B,\sigma}^{1/2}r^{1/2}\mathsf{g}\left(\frac{1}{m} + \frac{1+q}{|P|n}\right)^{1/4},
\end{equation}
where the first expectation is taken over the $m$ realizations of the dataset of patches $(\yp^{(i)},\xp^{(i)})_{1\leq i \leq m}$. 
\end{theorem}
\begin{proof}
The result corresponds to Thm.~4 of \cite{ciliberto19localized} where in the proof we used our comparison inequality \eqref{eq:comparisoninequality} instead of theirs.
\end{proof}
In order to prove our main result (\cref{th:generalizationbound}), we now show that the assumptions on the loss $\ell$ and the target distribution appearing in \cref{sec:stats} imply the assumptions made in \cref{th:generalizationboundcomplete}.

\subsection{Assumptions on the Loss $\boldsymbol{\ell}$ and Target Distribution} \label{app:assumptionslosstarget}
We recall here the assumption on the smoothness on $\ell$ and $\rho$ and derive some implications that will be useful to prove \cref{th:generalizationbound}.
Let the set of patches be defined respectively as $[\calX] = [-1,1]^{d_{\calX}\times d_{\calX}}$ and $[\calY] = \Rspace{d_{\calY}\times d_{\calY}}$, with $d_\calX, d_\calY \in \mathbb{N}$. Denote by $W^{s}_2(Z)$ the {\em Sobolev space} of smoothness $s > 0$ in a set $Z$, i.e. the space of functions with square-integrable weak derivatives up to order $s$ \cite{adams2003sobolev}.

\noindent\textbf{Assumption on the loss $\boldsymbol{\ell}$.} There exists $C \in (0,\infty), s \geq \frac{d_{\calX}^2 +1}{2}$ s.t. $\sup_{\xp \in [\calX]} \|\ell(\cdot,\xp)\|_{W^{s}_2([\calX])} \leq C$.

\begin{lemma}\label{lm:cond-on-ell}
Under the assumption above, $\ell$ satisfies the SELF assumption with $\calH$ = $W^{p}_2([\calX])$, with $p = \frac{d_\calX^2 + 1}{2}$ and $\psi$ be the feature map associated to the Abel kernel, i.e., $\left\langle\psi(\xp),\psi(\xp')\right\rangle_{\calH} = e^{-\|\xp - \xp'\|}$.
\end{lemma}
\begin{proof}
Since $s \geq p$, we have $W^{s}_2([\calX]) \subseteq W^{p}_2([\calX])$ \cite{adams2003sobolev}.
By Theorem 8 point (c) \cite{ciliberto20general} and the assumption on $\ell$, we have that there exists a uniformly bounded measurable feature map $\phi$, such that $\ell(\xp,\xp') = \left\langle\psi(\xp),\phi(\xp')\right\rangle_{\calH}$, with $\calH = W^{p}_2([\calX])$.
\end{proof}

\noindent\textbf{Assumption on the target distribution.}
Let $m \in \N$ s.t. $W^m_2([\calY]) \subseteq {\calH}$ where $\calH$ is the reproducing kernel Hilbert space associated to the kernel $k$ \cite{aronszajn1950theory, berlinet2011reproducing}. For example, $m = d_{\calY}^2/2 + 1$ for the Laplacian kernel $k(\yp,\yp') = e^{-\|\yp-\yp'\|}$. 
We require $\rho(\xp|\yp)$ to be density and satisfy $\rho(\xp|\yp) \in W^{m}_2([\calX] \times [\calY]).$

\begin{lemma}\label{lm:cond-on-rho}
Under the assumption above on $\rho$ and $k$ we have $\|g^\star\|_{{\cal G} \otimes {\cal H}} < \infty$.
\end{lemma}
\begin{proof}
Note that by the SELF condition, $g^\star = \int \phi(\xp) d\rho(\xp|\yp)$, where $\phi: [\calX] \to {\cal H}$ is a uniformly bounded measurable map and ${\cal H}$ a separable Hilbert space defined in the proof of the lemma above.
Note moreover, that since $\rho \in W^{m}_2(\calX \times \calY),$ and $W^{m}_2(\calX \times \calY) \subseteq W^{m}_2(\calX) \otimes W^{m}_2(\calY)$ \cite{adams2003sobolev}. Finally note that $W^{m}_2(\calX) \otimes W^{m}_2(\calY) \subseteq W^{m}_2(\calX) \otimes {\cal G}$ since $W^{m}_2(\calY) \subseteq {\cal G}$ by assumption. Since $\rho \in W^{m}_2(\calX) \otimes {\cal G}$, then there exists a Hilbert-Schmidt operator $R: {\cal G} \to W^{m}_2(\calX)$ such that $\rho(\xp|\yp) =  \left\langle h_{\calX}(\xp), R h_{\calY}(\yp) \right\rangle_{W^{m}_2(\calX)}$, where $h_{\calY}: [\calY] \to {\cal G}$ is the uniformly bounded continuous feature map associated to the kernel $k$ and $h_{\calX}: [\calX] \to W^{m}_2(\calX)$ is the uniformly bounded continuous feature map associated to the canonical kernel of $W^{m}_2(\calX)$. Now note that by construction $\rho$ is a density and so, for all $\yp \in [\calY]$, we have
\begin{align*}
g^\star(\yp) &= \int_\calX \phi(\xp) \left\langle h_{\calX}(\xp), R h_{\calY}(\yp) \right\rangle_{W^{m}_2(\calX)} d\xp \\
& = \left(\int_\calX \phi(\xp) \otimes h_{\calX}(\xp) d \xp \right) ~R~ h_{\calY}(\yp) = ~~ g^\star h_{\calY}(\yp),
\end{align*}
where $g^\star = H R$, with $H: W^{m}_2(\calX) \to {\cal H}$ is a linear operator defined as $H = \int_\calX \phi(\xp) \otimes h_{\calX}(\xp) d \xp$. Now note that $H$ is trace class, since both $\phi, h_{\calX}$ are uniformly bounded and $\calX$ is compact. Since $R$ is a Hilbert-Schmidt opertator, then $\|g^\star\|_{HS} \leq \|H\|_{op} \|R\|_{HS} \leq \|g^\star\|_{HS} \leq \|H\|_{tr} \|R\|_{HS} < \infty$.
The proof is concluded by considering the isomorphism between the Hilbert space ${\cal G} \otimes {\cal H}$ and the space of Hilbert-Schmidt operators between ${\cal G}$ and ${\cal H}$, since we have already proved that $g^\star(\yp) = g^\star h_{\calY}(\yp)$, for all $y \in [\calY]$, with $g^\star$ a Hilbert-Schmidt operator.
\end{proof}

\paragraph{Proof of \cref{th:generalizationbound}.} By \cref{lm:cond-on-ell} and \cref{lm:cond-on-rho} we first prove that the assumptions on $\ell$ and on $\rho$ satisfy the requirements on \cref{th:generalizationboundcomplete} and then we apply \cref{th:generalizationboundcomplete}.

\section{Algorithm} \label{app:algorithm}
In this section we discuss the algorithmic aspects of our estimator. In \cref{app:convexity} we study the convexity of the energy when the $\alpha$'s are computed using kernel ridge regression (KRR) and the Nadaraya-Watson (NW) estimator. In \cref{app:generalalgorithms} we expose linearly convergent methods to solve the $z_p$ update \eqref{eq:updatezp} and in \cref{app:sdcaeuclidean} we derive SDCA with gap sampling for the case when the $\alpha$'s are computed using the NW estimator and the loss $\ell$ is the Euclidean loss.

\subsection{Convexity of the Energy} \label{app:convexity}

For the convexity analysis, we will assume for simplicity that $q\approx 0$, which means that there is small correlation between the patches of the degraded image. This is further justified in \cref{app:constantq}.

We study the convexity of the energy of the $z_p$ update \eqref{eq:updatezp}:
\begin{equation}\label{app:equpdatezp}
    \sum_{i=1}^m \alpha_i(\yp)\ell(\zp, \xp^{(i)}) + \beta/2 \|\zp - \xp\|_2^2,
\end{equation}
when $\ell$ is convex in its first argument, 
where for the sake of exposition, we have defined $\yp=y_p$, $\zp\defeq z_p$ and $\xp=x_p$.
In the case where the $\alpha$'s are computed using the Nadaraya-Watson (NW) estimator, all the coefficients are positive, and so \eqref{app:equpdatezp} is convex because $\alpha_i(\yp)\ell(\zp, \xp^{(i)})$ remain convex.
Unfortunately, some coefficients may be non-positive when the $\alpha$'s are computed using kernel ridge regression (KRR). The goal of the following \cref{th:convexityenergy} is to show that in this case, in expectation the energy is also convex.

\begin{theorem}[Convexity of the energy]\label{th:convexityenergy} Assume that the loss $\ell$ is two-times (absolutely) continuously differentiable in the first argument and the target distribution satisfies the same assumption as in \cref{th:generalizationbound}. Let $\lambda = rm^{-1/2}$. If $\ell$ is convex in the first argument, there exists $m'\in\mathbb{N}$ such that
\begin{equation}\label{app:equpdatezpth}
    \Expect~\Expect_{\yp\sim\rho_{\calY}}\sum_{i=1}^m \alpha_i(\yp)\ell(\zp, \xp^{(i)}) + \beta/2 \|\zp - \xp\|_2^2,
\end{equation}
is convex in $z\in[\calX]$ for all~$m\geq m'$, where the first expectation is taken over the $m$ realizations of the dataset of patches.
\end{theorem}
\begin{proof}
Let $\ell_{jk}(\xp, \xp')$ be the $j,k$-crossed-derivatives of the loss $\ell$ in the first argument. From \cite{ciliberto2016consistent} (see Thm. 19), we know that if $\ell_{jk}(\xp, \xp')$ is absolutely continuous, then the SELF assumption \eqref{eq:SELFapp} is satisfied and we can write
\begin{equation*}
    \ell_{jk}(\xp,\xp') = \langle\psi_{jk}(\xp), \phi_{jk}(\xp')\rangle_{\calH_{jk}}.
\end{equation*}
for continuous mappings $\psi_{jk}, \phi_{jk}:[\calX]\rightarrow\calH_{jk}$ and separable Hilbert space $\calH_{jk}$. Let us now define $H_{\zp}(\yp)$ and $\widehat{H}_{\zp}(\yp)$ to be the Hessians of $\zp\mapsto\Expect_{\xp'\sim\pi_{\yp}}\ell(\zp, \xp')$ and $\zp\mapsto \sum_{i=1}^m \alpha_i(\yp)\ell(\zp, \xp^{(i)})$.
If we denote $g_{jk}^\star(\yp) = \Expect_{\xp'\sim\rho_{\yp}}\phi_{jk}(\xp')$ and 
$\widehat{g}_{jk}(\yp) = \sum_{i=1}^m \alpha_i(\yp)\phi_{jk}(\xp^{(i)})$, we have that
\begin{align}
    \sup_{\zp\in[\calX]}\Expect \Expect_{\yp}(\widehat{H}_{\zp}^{jk}(\yp) - H_{\zp}^{jk}(\yp))^2 &= \sup_{\zp\in[\calX]} \Expect\Expect_{\yp}\big(\langle\psi_{jk}(\zp), \widehat{g}_{jk}(\yp) - g^\star_{jk}(\yp)\rangle_{\calH_{jk}}\big)^2 \\
    & \leq \sup_{\zp\in[\calX]}\|\psi_{jk}(\zp)\|_{\calH_{jk}}^2\Expect\Expect_{\yp}\|\widehat{g}_{jk}(\yp) - g^\star_{jk}(\yp)\|_{\calH_{jk}}^2 \\
    &\leq \sup_{\zp\in[\calX]} C\|\psi_{jk}(\zp)\|_{\calH_{jk}}^2\|g_{jk}^\star\|_{\calG\otimes\calH_{jk}}^2m^{-1/2} \label{eq:krrpasseq} \\
    &= C'\|g_{jk}^\star\|_{\calG\otimes\calH_{jk}}^2m^{-1/2} \label{eq:lastconv}.
\end{align}
In \cref{eq:krrpasseq}, we have used the classical finite-sample generalization bound for vector-valued KRR $\Expect\Expect_{\yp}\|\widehat{g}_{jk}(\yp) - g^\star_{jk}(\yp)\|_{\calH_{jk}}^2\leq C\|g_{jk}^\star\|_{\calG\otimes\calH_{jk}}^2m^{-1/2}$, which can be found in Appendix B.4 of \cite{ciliberto2016consistent}, and in \cref{eq:lastconv} we have used that $\sup_{\zp\in[\calX]}\|\psi_{jk}(\zp)\|_{\calH_{jk}}^2<\infty$ as $\psi_{jk}$ is continuous and $[\calX]$ is compact.
Moreover, using \cref{lm:cond-on-rho}, we know that $\|g_{jk}^\star\|_{\calH\otimes\calH_{jk}}<\infty$.
Hence, we can bound the expected risk of the Hessian estimator measured by the Frobenius norm
\begin{equation}\label{eq:boundfrobenius}
    \Expect_{\yp}\|\widehat{H}_{\zp}(\yp) - H_{\zp}(\yp)\|_F^2 = \sum_{j=1}^{d_{\calX}}\sum_{k=1}^{d_{\calX}}\Expect_{\yp}(\widehat{H}_{\zp}^{jk}(\yp) - H_{\zp}^{jk}(\yp))^2 \leq C''m^{-1/2},
\end{equation}
for all $z\in\calZ$, where $\|A\|_{F}$ denotes the Frobenius norm of a matrix $A$.
Let $(\widehat{\sigma}^j_{\zp}(\yp))_{j=1}^{d_{\calX}}$ and $(\sigma^j_{\zp}( \yp))_{j=1}^{d_{\calX}}$ be the eigenvalues of $\widehat{H}_{\zp}(\yp)$ and $H_{\zp}(\yp)$, respectively.
From \eqref{eq:boundfrobenius}, we obtain that for all $z\in[\calX]$,
\begin{align}
    \max_{j\in[d_{\calX}]}\Expect_{\yp}|\widehat{\sigma}^j_{\zp}(\yp) -\sigma^j_{\zp}( \yp)| &\leq \Expect_{\yp}\max_{j\in[d_{\calX}]}|\widehat{\sigma}^j_{\zp}(\yp) -\sigma^j_{\zp}( \yp)| \label{eq:weylsinequality} \\
    &\leq \Expect_{\yp}\|\widehat{H}_{\zp}(\yp) - H_{\zp}(\yp)\|_2 \\
    &\leq \Expect_{\yp}\|\widehat{H}_{\zp}(\yp) - H_{\zp}(\yp)\|_F \\
    &\leq \Expect_{\yp}\big(\|\widehat{H}_{\zp}(\yp) - H_{\zp}(\yp)\|_F^2\big)^{1/2} \\
    & \leq Cm^{-1/4},
\end{align}
where in \eqref{eq:weylsinequality} we have used Weyl's inequality (see Thm. 4.5.3 in \cite{vershynin2018high}) and $\|A\|_2$ denotes the operator norm of the matrix $A$.
Using that $\sigma^j_{\zp}(\yp) \geq 0$ $\rho_{\calY}$-almost surely for every $j$ and $\zp$ as $\ell$ is convex, we have that 
\begin{equation}
    \min_{j\in[d_{\calX}]}\Expect_{\yp}\widehat{\sigma}^j_{\zp}(\yp) \geq -Cm^{-1/4}
\end{equation}
for all $\zp\in[\calX]$. Hence, if we take $m'$ such that $Cm'^{-1/4}\leq\beta$, we obtain that the composite energy~\eqref{app:equpdatezpth} is convex.
\end{proof}

\subsection{General Algorithms to Solve the $\boldsymbol{z_p}$ Update} \label{app:generalalgorithms}
By simplifying the notation as in \eqref{app:equpdatezp} and removing the dependence on $\yp$ in the $\alpha$'s, the $z_p$ update~\eqref{eq:updatezp} takes the form
\begin{equation}\label{eq:primal}
    \argmin_{\zp}~ \sum_{i=1}^m f_i(\zp) + \beta h(\zp),
\end{equation}
with 
\begin{equation}
    h(\zp) = \frac{1}{2} \|\zp - \xp \|_2^2, \hspace{0.3cm}\text{and}\hspace{0.3cm}
    f_i(\zp) = \alpha_i\| \zp - \xp^{(i)} \|_2, \quad \forall i \in \{1,\dots,m\}.
\end{equation}
As exposed in the previous section, if the $\alpha$'s are computed using KRR we cannot ensure their positivity, which entails that not all $f_i$ are convex.

In addition to the variant of SDCA we use in the specific case of non-negative $\alpha$'s, we can use other solvers handling the more general scenario where the $\alpha$'s can have negative values.

\textbf{Full-gradient methods.} Full-gradient methods use the global convexity of the loss and ignore the fact that the individual losses are potentially non-convex. A notable example of these algorithms is the Newton-method based algorithm from \cite{marteau19newton}, which achieves linear convergence rates with a logarithmic dependence on the condition number, but with a cubic dependence on the dimension $d_{\calX}$ of $[\calX]$.

\textbf{Stochastic Methods.} A prominent example of these methods is the dual-free SDCA from \cite{shalev2016sdca}. The algorithm achieves linear convergence with an additive quadratic dependence on the condition number, instead of the linear dependence in the case that all individual losses are convex.

\subsection{SDCA for Euclidean Loss and Nadaraya-Watson Estimator} \label{app:sdcaeuclidean}

\subsubsection{Derivation of the Dual} \label{app:derivationdual}
We assume now the $\alpha$'s are all positive, thus, the functions $f_i$ are all convex. In this case, the dual of problem \eqref{eq:primal} reads
\begin{equation}
    \argmax_{\mu}~-\sum_{i=1}^mf_i^*(-\mu_i) - \beta h^* \big(\sum_{i}\mu_i / \eta\big),
\end{equation}
where $g^*(y) = \sup_{x}~x^\top y - g(x)$ denotes the Fenchel conjugate \cite{rockafellar1970convex} of the function $g$. Moreover, if strong duality holds, we have that $z^\star = \nabla h^\star(\sum_{i}\mu_i / \beta)$. The Fenchel conjugates of the  functions $f_i$ and $h$ read
\begin{equation}
    f_i^\star (\mu_i) = x_i^\top \mu_i +  \mathrm{1}_{\mu_i \in \alpha_i \mathcal{B}_2}, \hspace{0.3cm}\text{and}\hspace{0.3cm}h^\star(\nu)= \frac{1}{2} \| \nu \|_2^2 + \nu^\top x.
\end{equation}
Hence, the dual problem reads
\begin{equation}
    \max_{\mu_i \in \alpha_i \mathcal{B}_2} - \beta \left\{ \frac{1}{2} \Big\| \smui/\beta \Big\|^2_2 + \Big( \smui/\beta \Big)^\top \xp \right\} + \smui^\top \xp^{(i)}.
\end{equation}
In our case, strong duality holds. Thus we can write the relation between primal and dual variables as follows
\begin{equation}\label{eq:primaldual}
    \zp^\star = \frac{\sum_i^m \mu_i^\star}{\beta} + \xp. 
\end{equation}

\subsubsection{Stochastic Dual Coordinate Ascent (SDCA)} \label{app:sdca}

We apply stochastic dual coordinate ascent \cite{shalev2013stochastic} to solve the problem. At each iteration, the algorithm picks an index $i\in[n]$ (at random or following a certain schedule) and performs maximization of the $i$-th dual variable $\mu_i$ as
\begin{equation}
   \mu_i^{(t+1)} \in \argmax_{\mu_i\in\alpha_i\mathcal{B}_2}~ -\frac{1}{2\beta}\|\mu_i\|^2 - \mu_i^\top\Big(\sum_{j\neq i}\mu_j^{(t)}/\beta + \xp  -\xp^{(i)}\Big).
\end{equation}
Note that $\sum_{j\neq i}\mu_j^{(t)}/\beta = \zp^{(t)} - \mu_i^{(t)}/\beta - \xp$, thus with
$b^{(t)} = \zp^{(t)} - \mu_i^{(t)}/\beta -\xp^{(i)}$, we have that
\begin{equation}
   \mu_i^{(t+1)} = -b_i^{(t)}\min(\alpha_i /\|b_i^{(t)}\|_2, \beta) = \argmax_{\mu_i\in\alpha_i\mathcal{B}_2}~ -\frac{1}{2\beta}\|\mu_i\|^2 - \mu_i^\top b_i^{(t)},
\end{equation}
which is equivalent to
\begin{equation}
     \mu_i^{(t+1)} = \left\{
     \begin{array}{ll}
         -\alpha_i b_i / \|b_i\|_2 & \text{if } \beta\|b_i\|_2 > \alpha_i,  \\
          -\beta b_i & \text{if }\beta\|b_i\|_2 < \alpha_i.
     \end{array}
     \right.
\end{equation}

Finally, we update $\zp$ with $\zp^{(t+1)} = \zp^{(t)} + (\mu_i^{(t+1)} - \mu_i^{(t)})/\beta$.

\subsubsection{Computation of the dual gap.}
The dual gap is the difference between the primal objective and the dual objective. Hence, it is a proxy for the primal error and it can then be used a stopping criterion.
Using the relation $\zp-\xp = \sum_{i=1}^m\mu_i^{(t)}/\beta$, the dual gap $g(\zp, \mu)$ takes the general form 
\begin{align*}
    g(\zp, \mu) &= \beta (h(\zp) + h^*(\zp-\xp)) + \sum_{i=1}^mf_i(\zp) + f_i^*(-\mu_i) \\
    &= \beta (\|\zp-\xp\|_2^2 + (\zp-\xp)^\top \xp) + \sum_{i=1}^m\alpha_i\|\zp-\xp^{(i)}\|_2 -\mu_i^\top \xp^{(i)} \\
    &= \beta (\zp-\xp)^\top \zp + \sum_{i=1}^n\alpha_i\|\zp-\xp^{(i)}\|_2 -\mu_i^\top \xp^{(i)} \\
    &= \sum_{i=1}^m (\zp-\xp)^\top \mu_i + \sum_{i=1}^m\alpha_i\|\zp-\xp^{(i)}\|_2 -\mu_i^\top \xp^{(i)} \\
    &= \sum_{i=1}^m \left\{\alpha_i\|\zp-\xp^{(i)}\|_2 + \mu_i^\top(\zp - \xp^{(i)})\right\} \\
    &= \sum_{i=1}^mg_i(\zp, \mu_i),
\end{align*}
where we have defined $g_i(\zp, \mu_i)\geq 0$ as the individual dual gaps.
With this particular formulation, the dual gap can be decomposed into individual gaps that we use for efficient gap sampling.

\textbf{Speeding up SDCA with gap sampling.} 
One question arising from the SDCA framework is how to select an index $i$ at each step. 
We make use a non-uniform gap sampling strategy proposed by \cite{le2018adaptive} and initially motivated by the analysis of \cite{osokin2016minding}. 
The idea is to sample from the distribution given by $(g_i(\zp^k, \mu_i^k))_{i=1}^m / g(\zp^k, \mu^k)$, i.e., proportional to the magnitude of the individual dual gaps. 
Intuitively, observations $i$ that have large dual gaps need to be processed with higher priority by the algorithm in order to rapidly minimize the global dual gap, which is a proxy of the primal (and also dual) error. 
We detail in \cref{alg:appsdcaeuclidean} the proposed version of SDCA with gap sampling. 
Computing the dual gap requires observing the full dataset, which has a complexity linear in the number of samples $m$, and thus can be time-consuming.
In practice, we do not compute the dual gap at every iteration but after a certain number of iterations, e.g., every 25 iterations in our implementation.

\begin{algorithm}[t]
 \caption{SDCA for solving \eqref{eq:updatezp} with Euclidean loss and NW estimator.}
\label{alg:appsdcaeuclidean}
\begin{algorithmic}
\STATE $\zp^1 = \xp + 1/\beta\sum_{i=1}^m\mu_i^{(0)}$ \;
\FOR{$k = 1, \dots, K$}
 \STATE Sample $i\sim (g_i(\zp^k, \mu_i^k))_{i=1}^m / g(\zp^k, \mu^k) $ \;
 \STATE $b_i^{k} = \zp^{k} - \mu_i^{k}/\beta - \xp^{(i)}$ \;
\STATE $\mu_i^{k+1} = -b_i^{k}\min(\alpha_i/\|b_i^{k}\|_2, \beta)$ \;
\STATE $\zp^{k+1} = \zp^{k} + (\mu_i^{k+1} - \mu_i^{k})/\beta$ \;
\ENDFOR
\end{algorithmic}
\end{algorithm}

\section{Experiments} \label{app:experiments}
\subsection{Empirical analysis of the constant $\boldsymbol{q}$} \label{app:constantq}

In this section we show that the quantity $q$ defined in \cref{eq:constantqapp}, which computes the total correlation between degraded and clean patches, is much smaller than $|P|$. 
If this is the case, we know that the generalization bound given by \cref{th:generalizationbound} decreases with the number of patches $m$ and not with the number of images $n$, which confirms the fact that working at the patch level considerably reduced the amount of required training data.

To empirically validate it, we show that the correlation of a patch (measured with the kernel used in our experiments described in \cref{sec:experiments}), which uses DCT features, is high with patches of its neighborhood and small with patches far away. 
Looking at \cref{eq:constantqapp}, this implies that the constant $q$ is independent of the total number of patches $|P|$.

We generate $n=250$ blurry images of size $101 \times 101$ with the dataset of \cite{martin01database}.
For each image, we compute the correlation of the central patch with the other patches. 
We decompose the images of the toy dataset into $8 \times 8$ patches, leading to $|P|=8836$ patches per image. 
Figure \ref{fig:intralocalityimage} illustrates the average correlation coefficients for the central patch of the images for the 3 first iterations of HQS.
At the first iteration, the correlations are computed on the blurry image and the subsequent iterations on the current estimates of the sharp image.
One can notice a decay of the correlation coefficients below $1\%$ of the central correlation when we are outside a $5 \times 5$ window around the central patch for iterations 2 and 3.
For the first iteration, the decay is slower but patches outside a $5 \times 5$ window around the central patch have a correlation below $6 \%$.
This validates the between locality assumption of \cref{sec:stats}.

\begin{figure}[h]
    \centering
    \begin{tabular}{ccc}
        \begin{subfigure}{0.30\textwidth}\includegraphics[scale=0.30]{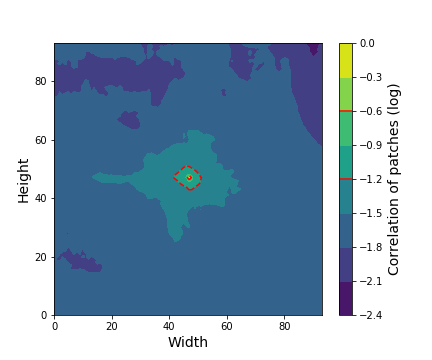}\caption*{Iteration 1.}\end{subfigure} &  
        \begin{subfigure}{0.30\textwidth}\includegraphics[scale=0.30]{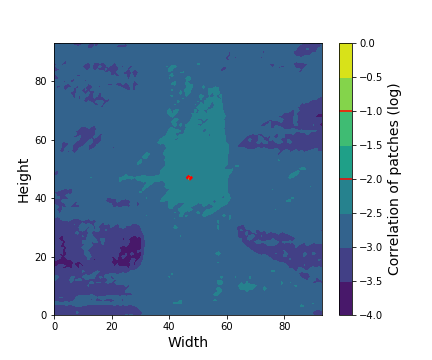}\caption*{Iteration 2.}\end{subfigure} &  
        \begin{subfigure}{0.30\textwidth}\includegraphics[scale=0.30]{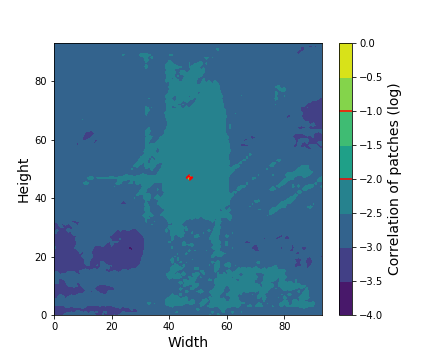}\caption*{Iteration 3.}\end{subfigure} \\
    \end{tabular}
    \caption{Correlation of the central patch of the image with the rest of patches. We use DCT features and a Gaussian kernel with the same bandwidth than in our experiments.}
    \label{fig:intralocalityimage}
\end{figure}

\subsection{Implementation Details} \label{app:implementationdetails}

In this section, we present more details on the setting used to run the experiments and give a canvas of our Python implementation of SDCA with gap sampling.

\subsubsection{Code of SDCA with gap sampling}

We present here a simplified code canvas in Python for implementing \cref{alg:sdcaeuclidean} with gap sampling, for clarity's sake.
We choose to use Pytorch to easily run the code on a GPU.

The variable ``xtr'' stands for the training images $\xp^{(i)}$ ($i=1,\dots,m$).
The variable ``gaps'' is a $m \times |P|$ matrix that contains the dual gap for each pair of training sample $\xp^{(i)}$ and test patch $x_p$.
Summing over the first axis yields the $|P|$ dual gaps for performing gap sampling.

\begin{lstlisting}[language=Python, caption=Python template for SDCA with gap sampling (\cref{alg:sdcaeuclidean}).]
import torch

def compute_dual_gaps(z, x, mu, alpha, xtr, beta):
    # mu: shape (m,P,ps,ps)
    # alpha: shape (m,P,1,1)
    alpha = alpha.squeeze(-1).squeeze(-1)
    z = z.unsqueeze(0)
    x = x.unsqueeze(0)
    xtr = xtr.unsqueeze(1)
    g = alpha*(z - xtr).norm(p=2, dim=(-1,-2))
    g += torch.einsum("ijkl,ijkl->ij", [mu, z-xtr])
    return g
    
def SDCA(x, alpha, xtr, beta, K, epsilon=1e-5):
    # x: shape (P, ps, ps): initial guess
    # alpha: shape (m, P, 1, 1)
    # xtr: (m, ps, ps)
    ## create the variables 
    n, P = alpha.shape[:2]
    ps = xtr.shape[-1]
    z = x
    mu = torch.zeros(m, P, ps, ps)
    ## compute first dual gaps
    gaps = compute_dual_gaps(z, x, mu, alpha, xtr, beta)
    ## do loop
    k = 0
    while gaps.sum(0).max() > epsilon and k < K:
        _, i = gaps.max(dim=0)  # gap sampling
        b = z - mu[i] / beta - xtr[i]
        nb = b.norm(p=2, dim=(-1,-2), keepdim=True)
        mu_i = -b * torch.min(alpha[i] / nb, beta)
        z += (mu_i - mu[i]) / beta
        mu[i] = mu_i
        # update all the gaps
        gaps = compute_dual_gaps(z, x, mu, alpha, xtr, beta)
        k += 1
    return z
\end{lstlisting}

\subsubsection{Additional details}
We run the code on a NVIDIA V100 GPU with 32Gb of graphic memory.

One crucial element of our implementation is the design of the bandwidth $\sigma_k$ of the Gaussian kernel $k$.
We compute the DCT features for the training patches $\xp^{(i)}$ ($i=1,\dots,m$), resulting in a $m \times d$ matrix (for $8 \times 8$ patches, $d=64$).
We then compute the standard deviations $\sigma^1, \dots, \sigma^d$, one for each row of the feature matrix.
Finally, we choose the bandwidth as $\sigma_k = s \| (\sigma^1, \dots, \sigma^d)\|_2$ with $s$ a scalar set to 0.2 in our implementation.

In practice, recomputing the dual gaps at each SDCA iterations is time-consuming.
Instead, we update the dual gaps every 25 SDCA steps to reduce the computation time of the update of $z$.
We have not noticed a drop of performance by doing so, compared to recomputing the dual gap at each iteration.

For deblurring, we run $T=8$ iterations of HQS in 6 minutes for a $321 \times 481$ image of BSD68 \cite{martin01database} with $m=10000$ training samples. It could be greatly reduced by selecting a smaller set of training samples, i.e., selecting a smaller $m$, and is one of our lines of research.
As a comparison, EPLL runs $T=6$ steps of HQS in about 3 minutes on our workstation for the same images.

\subsection{Further Experiments} \label{app:furtherexperiments}

In this section, we present an ablation study we have run to select images to collect training patches from and our current results and observations for denoising.

\subsubsection{Choice of datasets.}
We first evaluate the impact of the diversity of the images used to build our training set of patches.
We select 5 images depicting penguins, tigers, bears and buildings from the training set of BSD300 \cite{martin01database}, select one from each class as test image and use the four others as training images to sample $m$ training pairs $(\xp^{(i)}, \yp^{(i)})$.
The images are blurred with kernel 2 from \cite{krishnan09hyperlaplacian} and have $1\%$ additive Gaussian noise.
In \cref{tab:dataset_comparison}, the rows correspond to the test images to be deblurred an the columns are the compositions of the datasets.
We use the same experimental setting than in the main paper for deblurring.
``Penguin'', ``Tiger'', ``Bear'' and ``Building'' correspond to sets of 4 images of the same class and ``Hybrid'' is a set made of one image of each class.

\begin{table}[h]
\caption{Deblurring performance depending of the composition of the training set.}
    \centering
    \begin{tabular}{cccccc} \toprule
        Tr. classes & Penguin & Tiger & Bear & Building & Hybrid \\ \midrule
        Penguin & 29.64 & 29.51 & 29.31 & 28.68 & 29.69\\
        Tiger & 26.91 & 26.90 & 26.81 & 26.74 & 26.89\\
        Bear & 30.92 & 30.81 & 30.61 & 29.59 & 30.94\\
        Building & 24.44 & 24.49 & 24.46 & 24.81  & 24.49\\  \midrule
        Average & \textbf{27.98} & 27.93 & 27.80 & 27.46 & \textbf{28.00} \\
        \bottomrule
    \end{tabular}
    
    \label{tab:dataset_comparison}
\end{table}

In \cref{tab:dataset_comparison}, we show that when restoring an image with patches taken from images of the class, deblurring results are better. 
If now one restores a picture of an animal with images of another animal, performance are in the same ballpark, validating the fact that useful patches can be found in other images than the ones of the same class.
Now if one takes images of buildings to restore images of animals, the performance drops on the images of the penguin and the bear by -1dB.
A similar effect exists if one wants to restore an image of a building with images of animals where a drop of -0.4dB can be observed on the PSNR score the test image depicting a building.

This can be understood by the fact that images of buildings contain regular structures such as lines whereas images of animals feature landscapes and highly-textured surfaces such as furs and grass.
Restraining the training patches to the sole building structures is less effective to restore these details.

Considering three images of animals and one image of a building leads to marginally better results in average than using specific classes for training.
One can see that using only building images leads to a drop of -0.5dB in average compared to the other configurations, showing that variety of the training patches can have an important impact on the restoration performance.

\begin{figure}[t]
    \centering
    \adjustbox{max width=0.99\textwidth}{
    \begin{tabular}{cccc}
        \begin{subfigure}{0.25\textwidth}\includegraphics[scale=0.28]{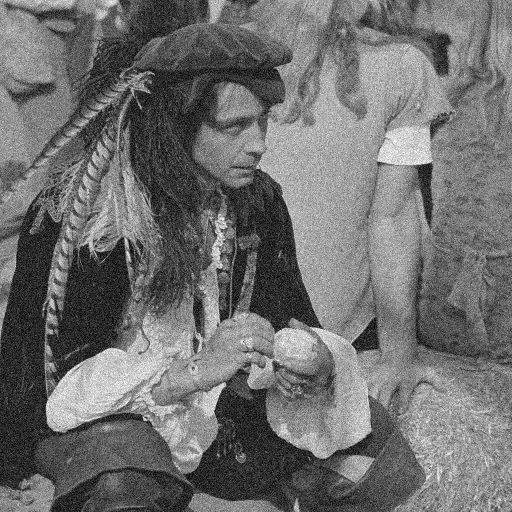}\caption*{Blurry image (21.51dB).}\end{subfigure} &  
        \begin{subfigure}{0.25\textwidth}\includegraphics[scale=0.20]{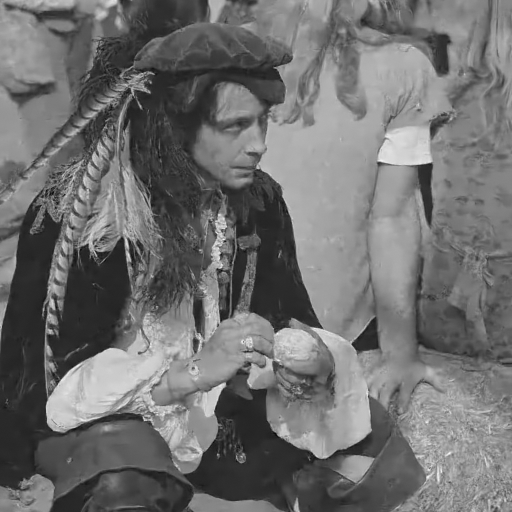}\caption*{EPLL (31.93dB).}\end{subfigure} & 
        \begin{subfigure}{0.25\textwidth}\includegraphics[scale=0.28]{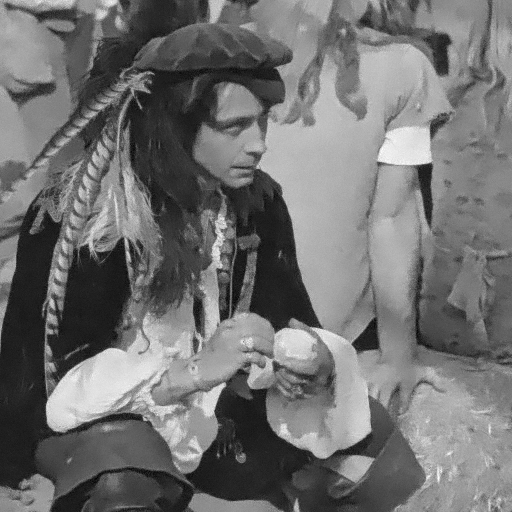}\caption*{Ours ($\ell_2$, 30.11dB).}\end{subfigure} &
        \begin{subfigure}{0.25\textwidth}\includegraphics[scale=0.28]{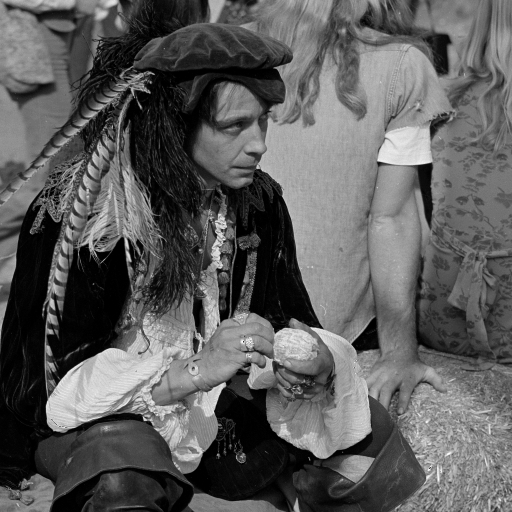}\caption*{Ground-truth image.}\end{subfigure} \\
    \end{tabular}
    }
    \caption{A denoising example from Set12 \cite{zhang17beyond}, with $\sigma=15$, better seen on a computer screen.
    Even though our methods removes most of the noise and produces a decent result, it oversmooths textures and keep residual noise at the salient edges of the image such as the hands of the man whereas EPLL filters out the noise and conserves the high-frequency details, explaining the gap of nearly 2dB with our restored image.}
    \label{fig:noise_man}
\end{figure}

\subsubsection{Denoising} We use the classical dataset of 12 images used in \cite{zhang17beyond} for denoising evaluation.
We add Gaussian noise of standard deviation 15, 25 and 50 (for pixel values in [0,255]). 
We show comparisons with a sparse coding technique \cite{elad06image}, a non-local means approach \cite{dabov07image} and a GMM-based prior \cite{zoran11learning} in \cref{tab:noise}.
We set $T=5$, $\delta=2$, $\beta_0=0.015$ and $\gamma=64,32$ and 16 for $\sigma=15,25$ and 50.

\begin{table}[h]
\caption{Average PSNR for denoising.}
    \centering
    \begin{tabular}{cccccc} \toprule
        Methods & KSVD \cite{elad06image} & BM3D \cite{dabov07image} & EPLL \cite{zoran11learning} & Ours~($\ell_2$) \\ \midrule
        $\sigma=15$ & 31.98 & \textbf{32.36} & 32.10 & 29.39 \\
        $\sigma=25$ & 29.33 & \textbf{29.91} & 29.62 & 27.31 \\
        $\sigma=50$ & 24.82 & 26.20 & \textbf{26.32} & 24.44 \\ \bottomrule
    \end{tabular}
    
    \label{tab:noise}
\end{table}

We are clearly below standard variational methods for the task of denoising by margins of about -2dB to -3db for each considered noise level in terms of PSNR.
This might be explained by the fact that denoising is a less structured problem than deblurring or upsampling, i.e., no information can be exploited from a forward operator $B$ that is in this case the identity.
Thus, one must carefully design the coefficients $\alpha$ to filter out some noise during reconstruction.
The DCT features appear to not be robust enough to noise and we will explore different features to circumvent this issue and improve the denoising performance of the proposed framework.
A denoising example is given in \cref{fig:noise_man} where one can observe that most of the noise is removed and artifacts appear at the edges of the elements of the image such as the fingers of the man.

\section{Additional Images} \label{app:additionalimages}

In this section, we present additional images for non-blind deblurring and upsampling. 
The images are better seen on a computer screen.
The metric is the PSNR.

\subsection{Non-blind deblurring}

In this section we present more non-blind deblurring examples comparing EPLL and the $\ell_2$ version of the proposed framework on the test set of BSD \cite{martin01database} with the same experimental setting than presented in the main paper. 
The results are shown in Figures \ref{fig:nbd_islander}, \ref{fig:nbd_castle}, \ref{fig:nbd_ruins}, \ref{fig:nbd_fish} and \ref{fig:nbd_tiger}.
This demonstrates the practical ability of structured prediction to handle this task as well as a standard variational framework, i.e., EPLL \cite{zoran11learning}, with a prior learned over 200000 patches.
\begin{figure}[!ht]
    \centering
    \adjustbox{max width=0.99\textwidth}{
    \begin{tabular}{cccc}
        \begin{subfigure}{0.25\textwidth}\includegraphics[scale=0.47]{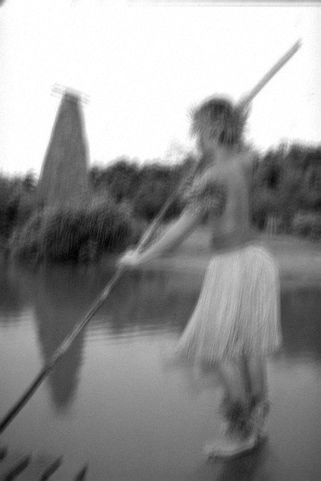}\caption*{Blurry image (24.28dB).}\end{subfigure} &  
        \begin{subfigure}{0.25\textwidth}\includegraphics[scale=0.34]{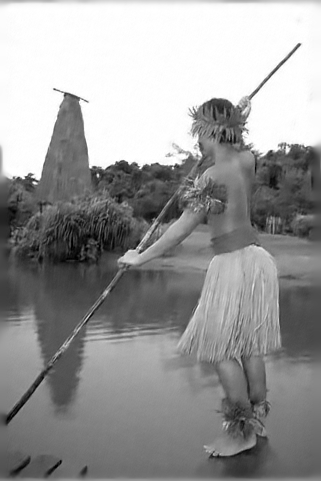}\caption*{EPLL (28.10dB).}\end{subfigure} & 
        \begin{subfigure}{0.25\textwidth}\includegraphics[scale=0.47]{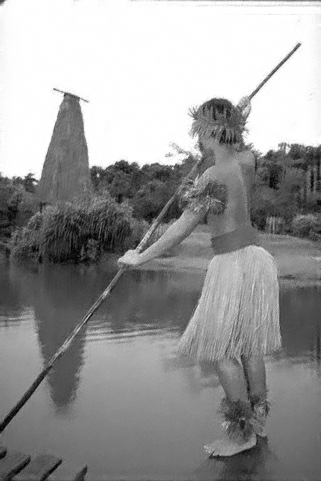}\caption*{Ours ($\ell_2$, 29.60dB).}\end{subfigure} &
        \begin{subfigure}{0.25\textwidth}\includegraphics[scale=0.47]{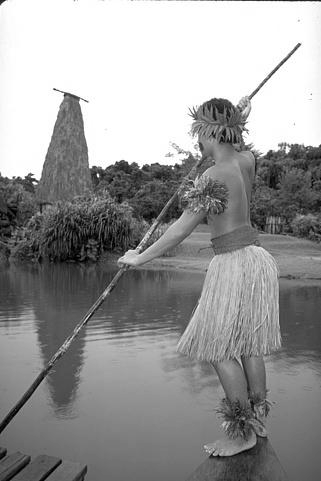}\caption*{Ground-truth image.}\end{subfigure} \\
    \end{tabular}
    }
    \caption{A non-blind deblurring example from BSD68 \cite{martin01database}.}
    \label{fig:nbd_islander}
\end{figure}
\begin{figure}[!ht]
    \centering
    \adjustbox{max width=0.99\textwidth}{
    \begin{tabular}{cccc}
        \begin{subfigure}{0.25\textwidth}\includegraphics[scale=0.47]{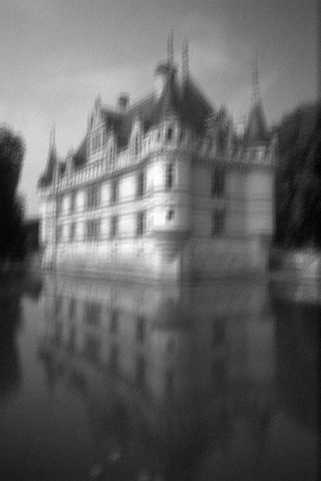}\caption*{Blurry image (24.37dB).}\end{subfigure} &  
        \begin{subfigure}{0.25\textwidth}\includegraphics[scale=0.34]{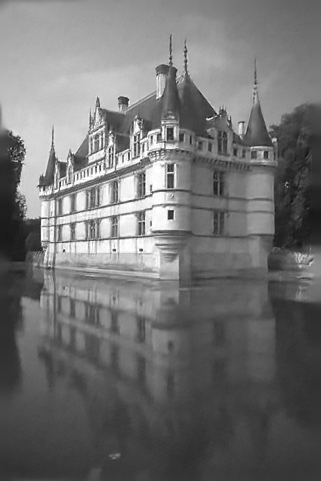}\caption*{EPLL (29.84dB).}\end{subfigure} & 
        \begin{subfigure}{0.25\textwidth}\includegraphics[scale=0.47]{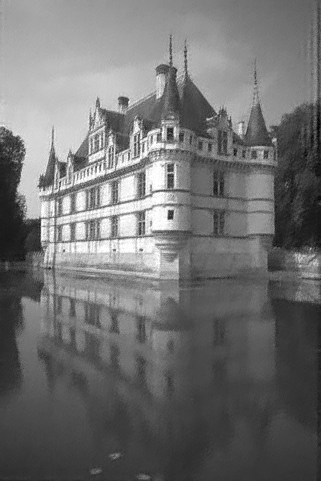}\caption*{Ours ($\ell_2$, 28.89dB).}\end{subfigure} &
        \begin{subfigure}{0.25\textwidth}\includegraphics[scale=0.47]{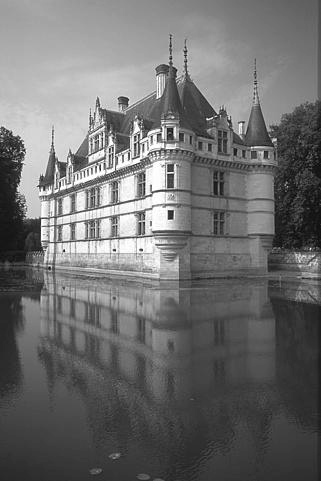}\caption*{Ground-truth image.}\end{subfigure} \\
    \end{tabular}
    }
    \caption{A non-blind deblurring example from BSD68 \cite{martin01database}.}
    \label{fig:nbd_castle}
\end{figure}
\begin{figure}[!ht]
    \centering
    \adjustbox{max width=0.99\textwidth}{
    \begin{tabular}{cccc}
        \begin{subfigure}{0.25\textwidth}\includegraphics[scale=0.47]{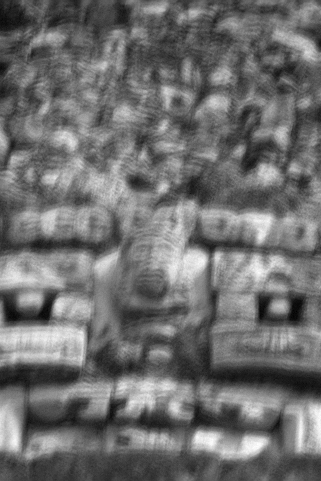}\caption*{Blurry image (21.51dB).}\end{subfigure} &  
        \begin{subfigure}{0.25\textwidth}\includegraphics[scale=0.34]{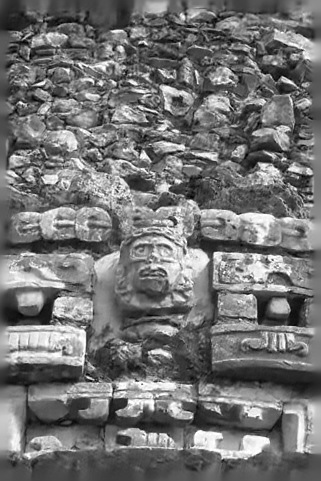}\caption*{EPLL (22.76dB).}\end{subfigure} & 
        \begin{subfigure}{0.25\textwidth}\includegraphics[scale=0.47]{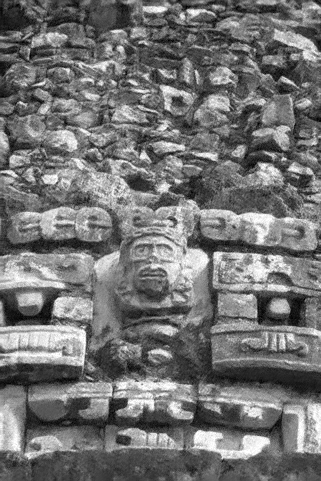}\caption*{Ours ($\ell_2$, 23.65dB).}\end{subfigure} &
        \begin{subfigure}{0.25\textwidth}\includegraphics[scale=0.47]{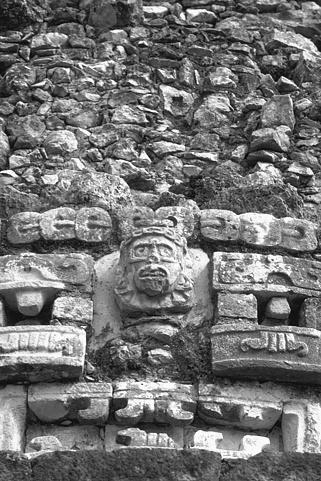}\caption*{Ground-truth image.}\end{subfigure} \\
    \end{tabular}
    }
    \caption{A non-blind deblurring example from BSD68 \cite{martin01database}.}
    \label{fig:nbd_ruins}
\end{figure}
\begin{figure}[!ht]
    \centering
    \adjustbox{max width=0.99\textwidth}{
    \begin{tabular}{cccc}
        \begin{subfigure}{0.25\textwidth}\includegraphics[scale=0.47]{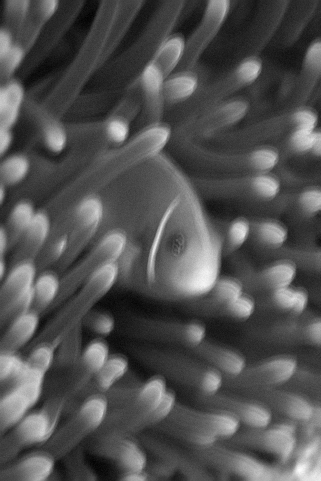}\caption*{Blurry image (24.07dB).}\end{subfigure} &  
        \begin{subfigure}{0.25\textwidth}\includegraphics[scale=0.34]{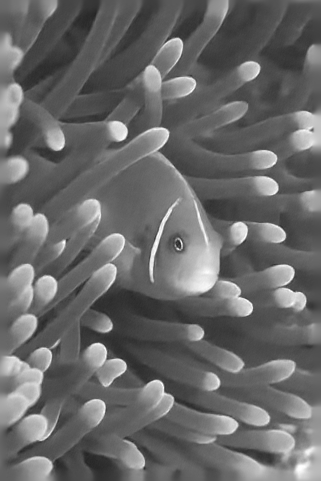}\caption*{EPLL (29.49dB).}\end{subfigure} & 
        \begin{subfigure}{0.25\textwidth}\includegraphics[scale=0.47]{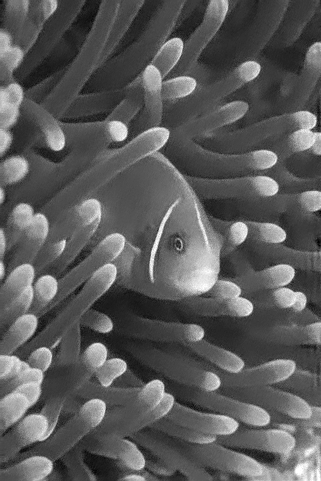}\caption*{Ours ($\ell_2$, 31.71dB).}\end{subfigure} &
        \begin{subfigure}{0.25\textwidth}\includegraphics[scale=0.47]{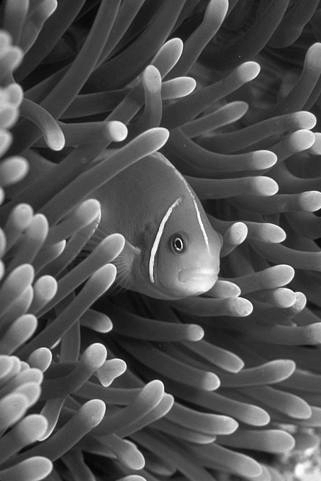}\caption*{Ground-truth image.}\end{subfigure} \\
    \end{tabular}
    }
    \caption{A non-blind deblurring example from BSD68 \cite{martin01database}.}
    \label{fig:nbd_fish}
\end{figure}
\begin{figure}[!ht]
    \centering
    \adjustbox{max width=0.99\textwidth}{
    \begin{tabular}{cc}
        \begin{subfigure}{0.50\textwidth}\includegraphics[scale=0.57]{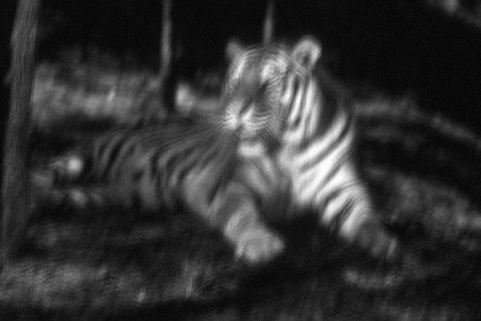}\caption*{Blurry image (24.60dB).}\end{subfigure} &  
        \begin{subfigure}{0.50\textwidth}\includegraphics[scale=0.412]{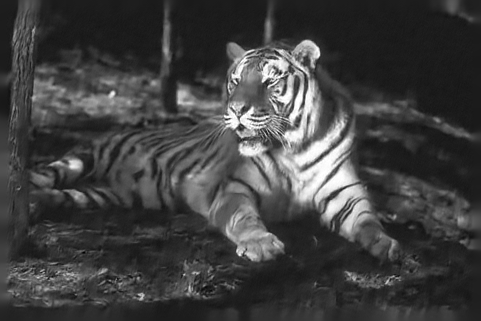}\caption*{EPLL (29.20dB).}\end{subfigure} \\
        \begin{subfigure}{0.50\textwidth}\includegraphics[scale=0.57]{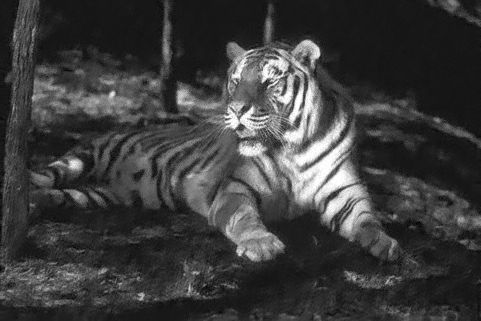}\caption*{Ours ($\ell_2$, 28.77dB).}\end{subfigure} &
        \begin{subfigure}{0.50\textwidth}\includegraphics[scale=0.57]{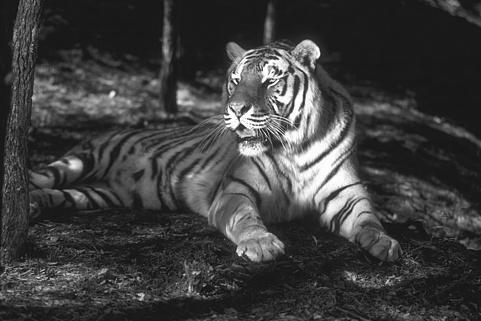}\caption*{Ground-truth image.}\end{subfigure} \\
    \end{tabular}
    }
    \caption{A non-blind deblurring example from BSD68 \cite{martin01database}.}
    \label{fig:nbd_tiger}
\end{figure}

\subsection{Upsampling}

In this section, we present qualitative results for upsampling by a factor $\times 2$ by comparing the $\ell_2$ version of our method with an image obtained by bicubic interpolation, KSVD \cite{zeyde10scaleup}, GR and ANR, two variants of \cite{timofte13anchored}. 
The images are taken from the Set5 dataset \cite{timofte13anchored} and we followed the same experimental protocol than in the paper.
The results are shown in Figures \ref{fig:sisr_bird}, \ref{fig:sisr_butterfly} and \ref{fig:sisr_woman}.
For each image we do better than bicubic interpolation but fail behind ANR and KSVD.
The differences explaining are results below other standard dictionary learning methdods are aliasing artifacts on the beak of the bird in \cref{fig:sisr_bird}, the black edges on the wings of the butterfly \cref{fig:sisr_butterfly} and the fingeres and sleeves of the woman \cref{fig:sisr_woman}.
These details are better seen on a computer screen.

\begin{figure}[!ht]
    \centering
    \adjustbox{max width=0.99\textwidth}{
    \begin{tabular}{ccc}
        \begin{subfigure}{0.33\textwidth}\includegraphics[scale=0.45]{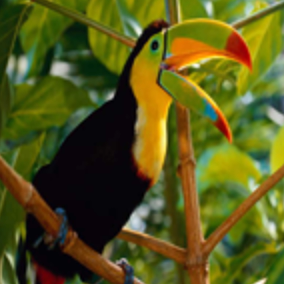}\caption*{Bicubic (36.8dB).}\end{subfigure} &  
        \begin{subfigure}{0.33\textwidth}\includegraphics[scale=0.45]{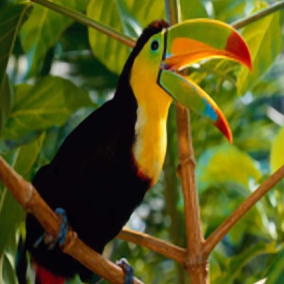}\caption*{KSVD (39.9dB).}\end{subfigure} & 
        \begin{subfigure}{0.33\textwidth}\includegraphics[scale=0.45]{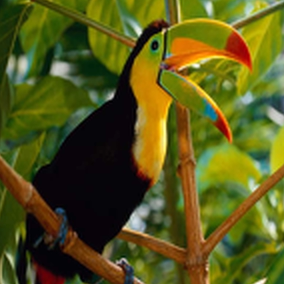}\caption*{GR (39.0dB).}\end{subfigure} \\ 
        \begin{subfigure}{0.33\textwidth}\includegraphics[scale=0.45]{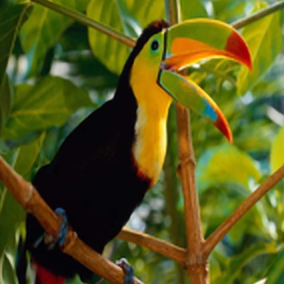}\caption*{ANR (40.0dB).}\end{subfigure} &
        \begin{subfigure}{0.33\textwidth}\includegraphics[scale=0.61]{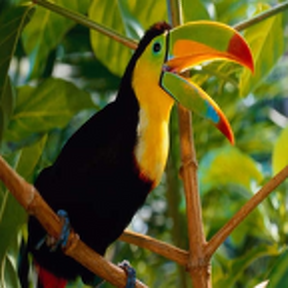}\caption*{Ours ($\ell_2$, 38.9dB).}\end{subfigure} &  
        \begin{subfigure}{0.33\textwidth}\includegraphics[scale=0.45]{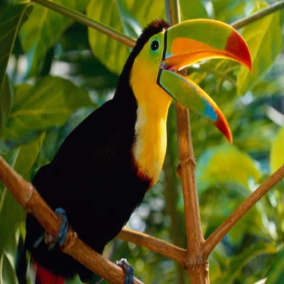}\caption*{Ground-truth image.}\end{subfigure} \\
    \end{tabular}
    }
    \caption{An upsampling example from Set5 \cite{timofte13anchored}.}
    \label{fig:sisr_bird}
\end{figure}

\begin{figure}[!ht]
    \centering
    \adjustbox{max width=0.99\textwidth}{
    \begin{tabular}{ccc}
        \begin{subfigure}{0.33\textwidth}\includegraphics[scale=0.52]{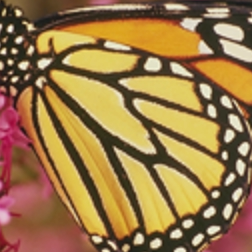}\caption*{Bicubic (27.4dB).}\end{subfigure} &  
        \begin{subfigure}{0.33\textwidth}\includegraphics[scale=0.52]{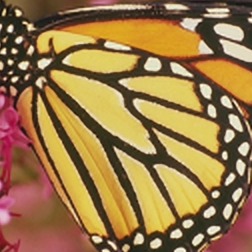}\caption*{KSVD (30.6dB).}\end{subfigure} & 
        \begin{subfigure}{0.33\textwidth}\includegraphics[scale=0.52]{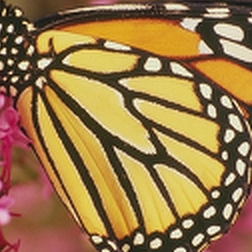}\caption*{GR (29.1dB).}\end{subfigure} \\ 
        \begin{subfigure}{0.33\textwidth}\includegraphics[scale=0.52]{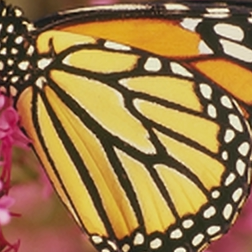}\caption*{ANR (30.5dB).}\end{subfigure} &
        \begin{subfigure}{0.33\textwidth}\includegraphics[scale=0.71]{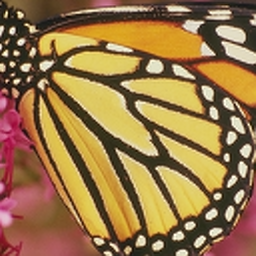}\caption*{Ours ($\ell_2$, 30.0dB).}\end{subfigure} &  
        \begin{subfigure}{0.33\textwidth}\includegraphics[scale=0.52]{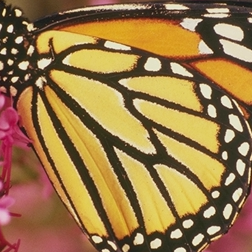}\caption*{Ground-truth image.}\end{subfigure} \\
    \end{tabular}
    }
    \caption{An upsampling example from Set5 \cite{timofte13anchored}.}
    \label{fig:sisr_butterfly}
\end{figure}

\begin{figure}[!ht]
    \centering
    \adjustbox{max width=0.99\textwidth}{
    \begin{tabular}{ccc}
        \begin{subfigure}{0.33\textwidth}\includegraphics[scale=0.54]{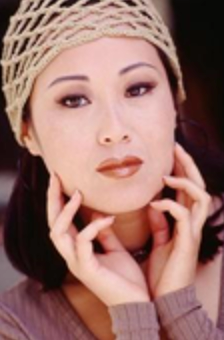}\caption*{Bicubic (32.1dB).}\end{subfigure} &  
        \begin{subfigure}{0.33\textwidth}\includegraphics[scale=0.54]{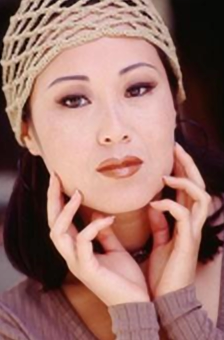}\caption*{KSVD (34.5dB).}\end{subfigure} & 
        \begin{subfigure}{0.33\textwidth}\includegraphics[scale=0.54]{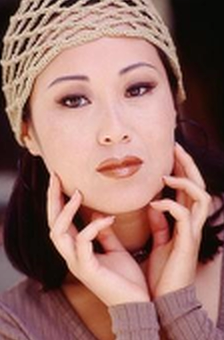}\caption*{GR (33.7dB).}\end{subfigure} \\ 
        \begin{subfigure}{0.33\textwidth}\includegraphics[scale=0.54]{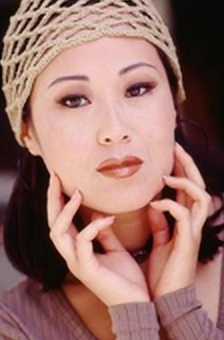}\caption*{ANR (34.5dB).}\end{subfigure} &
        \begin{subfigure}{0.33\textwidth}\includegraphics[scale=0.748]{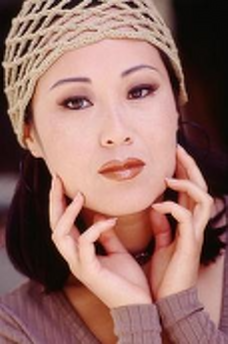}\caption*{Ours ($\ell_2$, 34.0dB).}\end{subfigure} &  
        \begin{subfigure}{0.33\textwidth}\includegraphics[scale=0.54]{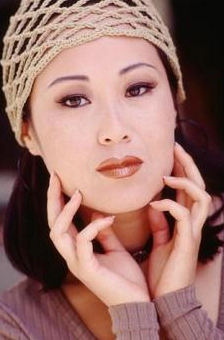}\caption*{Ground-truth image.}\end{subfigure} \\
    \end{tabular}
    }
    \caption{An upsampling example from Set5 \cite{timofte13anchored}.}
    \label{fig:sisr_woman}
\end{figure}
\end{document}